\documentclass[11pt]{article}
\usepackage{xspace,amsmath,amssymb,amsthm,amsfonts,epsfig,syntonly,dsfont}
\usepackage{cite,bm,color,url,textcomp}
\usepackage{algorithmicx,algorithm,algpseudocode}
\usepackage{epstopdf}
\usepackage{empheq,float}
\usepackage{graphicx,subfigure,balance}
\usepackage{multirow}
\newtheorem{assumption}{Assumption}

\newtheorem{lemma}{Lemma}
\newtheorem{corollary}{Corollary}
\newtheorem{theorem}{Theorem}

\theoremstyle{definition}\newtheorem{remark}{Remark}

\DeclareMathOperator*{\argmin}{arg\,min}
\usepackage{times}

\usepackage{anysize}            
\marginsize{0.7in}{0.7in}{0.27in}{1.27in} 

\begin{document}

\title{\huge Bandit Online Learning with Unknown Delays}
\author{Bingcong Li, Tianyi Chen, and Georgios B. Giannakis
\vspace{0.2cm} \\
\textit{University of Minnesota - Twin Cities, Minneapolis, MN 55455, USA} \\
\{\texttt{lixx5599,chen3827,georgios}\}@umn.edu
}

\maketitle

\begin{abstract}
This paper deals with bandit online learning problems involving feedback of unknown delay that can emerge in multi-armed bandit (MAB) and bandit convex optimization (BCO) settings. MAB and BCO require only values of the objective function involved that become available through feedback, and are used to estimate the gradient appearing in the corresponding iterative algorithms. 
Since the challenging case of feedback with \emph{unknown} delays prevents one from constructing the sought gradient estimates, existing MAB and BCO algorithms become intractable. For such challenging setups, delayed exploration, exploitation, and exponential (DEXP3) iterations, along with delayed bandit gradient descent (DBGD) iterations are developed for MAB and BCO, respectively.   
Leveraging a unified analysis framework, it is established that the regret of DEXP3 and DBGD are ${\cal O}\big( \sqrt{K\bar{d}(T+D)} \big)$ and ${\cal O}\big( \sqrt{K(T+D)} \big)$, respectively, where $\bar{d}$ is the maximum delay and $D$ denotes the delay accumulated over $T$ slots. Numerical tests using both synthetic and real data validate the performance of DEXP3 and DBGD.
\end{abstract}
\section{Introduction}

Sequential decision making emerges in various learning and optimization tasks, such as online advertisement, online routing, and portfolio management \cite{hazan2016,bubeck2012}. Among popular methods for sequential decision making, multi-armed bandit (MAB) and bandit convex optimization (BCO) have widely-appreciated merits because with limited information they offer quantifiable performance guarantees. MAB and BCO can be viewed as a repeated game between a possibly randomized learner, and the possibly adversarial nature. 
In each round, the learner selects an action, and incurs the associated loss that is returned by the nature. In contrast to the full information setting, only the loss of the performed action rather than the gradient of the loss function (or even the loss function itself) is revealed to the learner. Popular approaches to bandit online learning estimate gradients using several point-wise evaluations of the loss function, and use them to run online gradient-type iterative solvers; see e.g., \cite{auer2002a} for MAB and \cite{flaxman2005,agarwal2010} for BCO.

Although widely applicable with solid performance guarantees, standard MAB and BCO frameworks do not account for delayed feedback that is naturally present in various applications. For example, when carrying out machine learning tasks using distributed mobile devices (a setup referred to as federated learning) \cite{mcmahan2017}, delay comes from the time it takes to compute at mobile devices and also to transmit over the wireless communication links; in online recommendations the click-through rate could be aggregated and then periodically sent back \cite{li2010}; in online routing over communication networks, the latency of each routing decision can be revealed only after the packet's arrival to its destination \cite{awerbuch2004}; and in parallel computing by data centers,  computations are carried with outdated information because agents are not synchronized~\cite{agarwal2011,duchi2013,mcmahan2014}. 

Challenges arise naturally when dealing with bandit online learning with unknown delays, simply because unknown delayes prevent existing methods in non-stochastic MAB as well as BCO to construct reliable gradient estimates. To address this limitation, our solution is a fine-grained \emph{biased} gradient estimator for MAB and a \emph{deterministic} gradient estimator for BCO, where the standard unbiased loss estimator for non-stochastic MAB and the nearly unbiased one for BCO are no longer available. The resultant algorithms, that we abbreviate as DEXP3 and DBGD, are guaranteed to achieve ${\cal O}\big( \sqrt{K\bar{d}(T+D)} \big)$ and ${\cal O}\big( \sqrt{K(T+D)} \big)$ regret, respectively, over a $T$-slot time horizon with the maximum (overall) delay being $\bar{d}$ ($D$).

\subsection{Related works}
Delayed online learning can be categorized depending on whether the feedback information is full or bandit (meaning partial). We review prior works from these two aspects. 

\textbf{Delayed online learning.} This class deals with delayed but fully revealed loss information, namely fully known gradient or loss function. 
It is proved that an ${\cal O}\big(\sqrt{T+D}\big)$ regret for a $T$-slot time horizon with overall delay $D$ can be achieved. Particularly, algorithms dealing with a fixed delay have been studied in \cite{weinberger2002}. 
To reduce the storage and computation burden of \cite{weinberger2002}, an online gradient descent type algorithm for fixed $d$-slot delay was developed in \cite{langford2009}, where the lower bound ${\cal O}\big(\sqrt{(d+1)T} \big)$ was also provided. Adversarial delay has been tackled recently in \cite{joulani2016,shamir2017,quanrud2015}. However, the algorithms as well as the corresponding analyses in \cite{joulani2016,shamir2017,quanrud2015} are not applicable to bandit online learning setting when the delays are unknown.


\textbf{Delayed bandit online learning.} Stochastic MAB with delays has been reported in \cite{chapelle2011,desautels2014,vernade2017,pike2017}; see also \cite{joulani2013} for multi-instance generalizations introduced to handle adversarial delays in stochastic and non-stochastic MAB settings. For non-stochastic MAB, EXP3-based algorithms were developed to handle fixed 
delays in \cite{cesa2016, neu2010}.
Although not requiring memories for extra instances, the delay in \cite{cesa2016} and \cite{neu2010} must be known. A recent work \cite{cesa2018} considers a more general non-stochastic MAB setting, where the feedback is anonymous.\footnote{Anonymous feedback in MAB means that the learner observes the loss without knowing which arm it is associated with.} 

\subsection{Contributions}

Our main contributions can be summarized as follows.

\textbf{c1)} Based on a novel \emph{biased} gradient estimator, a delayed exploration-exploitation exponentially (DEXP3) weighted algorithm is developed for delayed non-stochastic MAB with \textit{unknown} and \textit{adversarially} chosen delays; 

\textbf{c2)} Relying on a novel \emph{deterministic} gradient estimator, a delayed bandit gradient descent (DBGD) algorithm is developed to handle the delayed BCO setting; and,
	
\textbf{c3)} A unifying analysis framework is developed to reveal that the regret of DEXP3 and DBGD is ${\cal O}\big( \sqrt{K\bar{d}(T+D)} \big)$ and ${\cal O}\big( \sqrt{K(T+D)} \big)$, respectively, where $\bar{d}$ is the maximum delay and $D$ denotes the delay accumulated over $T$ slots. Numerical tests validate the efficiency of DEXP3 and DBGD.

\textbf{Notational conventions}. Bold lowercase letters denote column vectors; $\mathbb{E}[\,\cdot\,]$ represents expectation; $\mathds{1}(\,\cdot\,)$ denotes the indicator function; $(\,\cdot\,)^{\top}$ stands for vector transposition; and $\| \bm{x}\|$ denotes the $\ell_2$-norm of a vector $\bm{x}$.

\section{Problem statements}

Before introducing the delayed bandit learning settings, we first revisit the standard non-stochastic MAB and BCO. 

\subsection{MAB and BCO}


\begin{figure}[t]
	\centering
	\includegraphics[height=0.45\textwidth]{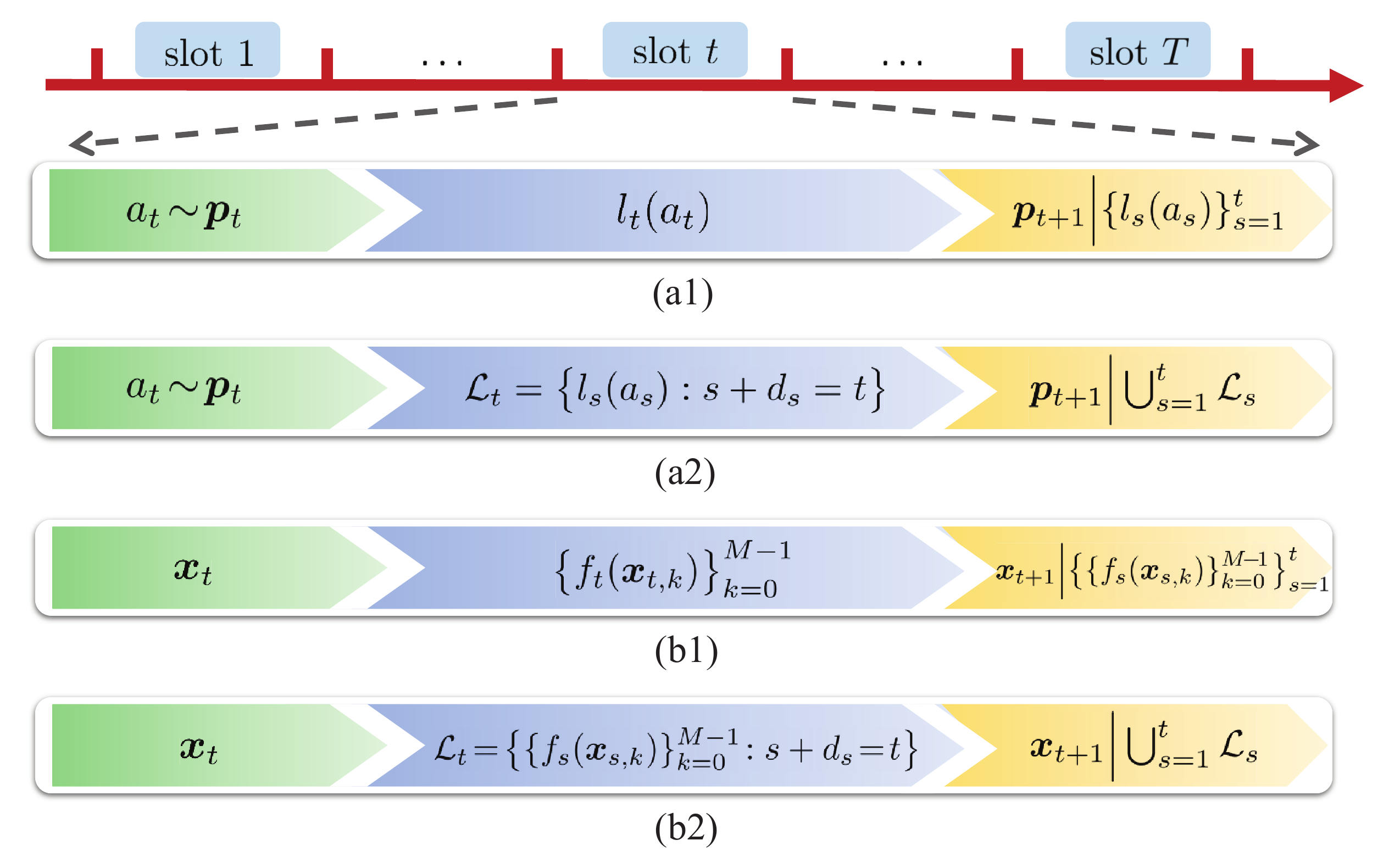}
	\vspace{-0.3cm}
	\caption{A single slot structure of: (a1) standard MAB; (a2) non-stochastic MAB with delayed feedback; (b1) standard BCO; (b2) BCO with delayed feedback, where $ \bm{p}_{t+1} | \{l_s(a_s)\}_{s=1}^t$ means that $\bm{p}_{t+1}$ updates based on previously observed losses $\{l_s(a_s)\}_{s=1}^t$.} 
	\label{fig.time_slot_sturcture}
	\vspace{-0.2cm}
\end{figure}

\textbf{Non-stochastic MAB.} 
Consider the MAB problem with a total of $K$ arms (a.k.a. actions) \cite{bubeck2012,auer2002a}. At the beginning of slot $t$, without knowing the loss of any arm, the learner selects an arm $a_t$ following a distribution $\bm{p}_t \in \Delta_K$ over all arms, where the probability simplex is defined as $\Delta_K:= \{ \bm{p} \in \Delta_K: p(k) \geq 0, \forall k;~ \sum_{k=1}^K p(k)=1\}$. The loss  $l_t (a_t)$ incurred by the selection of $a_t$ is an entry of the $K\times 1$ loss vector $\bm{l}_t$, and it is observed by the learner. Along with previously observed losses $\{ l_s(a_s) \}_{s=1}^t$, it then becomes possible to find $\bm{p}_{t+1}$; see also Fig. \ref{fig.time_slot_sturcture} (a1). 

The goal is to minimize the regret, which is the difference between the expected cumulative loss of the learner relative to the loss of the best fixed policy in hindsight, given by
\begin{equation}\label{eq.reg}
	\text{Reg}_T^{\text{MAB}} := \sum_{t=1}^T \mathbb{E} \big[\bm{p}_t^\top \bm{l}_t\big] - \sum_{t=1}^T (\bm{p}^*)^\top \bm{l}_t
\end{equation}
where the expectation is taken w.r.t. the possible randomness of $\bm{p}_t$ induced by the selection of $\{a_s\}_{s = 1}^{t-1}$, while the best fixed policy $\bm{p}^*$ is 
\begin{equation}
	\bm{p}^*: = \argmin_{\bm{p} \in \Delta_K} \sum_{t=1}^T \bm{p}^\top \bm{l}_t.
\end{equation}
Specifically, if $\bm{p}^* = [0,\ldots,1,\ldots,0]^\top$, the regret is relative to the corresponding best fixed arm in hindsight. 

\textbf{BCO.} Consider now the BCO setup with $M$-point feedback \cite{agarwal2010}. 
At the beginning of slot $t$, without knowing the loss, the learner selects $\bm{x}_t \in {\cal X}$, where ${\cal X} \subset \mathds{R}^K$ is a compact and convex set. Being able to query  the function values at another $M-1$ points $\{\bm{x}_{t,k} \in {\cal X}\}_{k = 1}^{M-1}$ and with $\bm{x}_{t,0} := \bm{x}_t$, the loss values at $\{\bm{x}_{t,k}\}_{k=0}^{M-1}$, that is, $\{f_t(\bm{x}_{t,k}) \}_{k=0}^{M-1}$, are observed instead of the function $f_t(\cdot)$. The learner leverages the revealed losses to decide the next action $\bm{x}_{t+1}$; see also Fig. \ref{fig.time_slot_sturcture} (b1). The learner's goal is to find a sequence of $\big{\{} \{\bm{x}_{t,k}\}_{k=0}^{M-1}\big{\}}_{t=1}^T$ to minimize the regret relative to the best fixed action in hindsight, meaning\footnote{This definition is slightly different with that in \cite{agarwal2010}. However, we will show in Section 5.2 that the regret bound is not affected.}
\begin{equation}\label{eq.reg2}
	\text{Reg}_T^{\text{BCO}} := \sum_{t=1}^T \mathbb{E} \big[ f_t(\bm{x}_t)\big] - \sum_{t=1}^T f_t(\bm{x}^*)
\end{equation}
where the expectation is taken over the sequence of random actions $\{\bm{x
}_{\tau}\}_{s=1}^{t-1}$. The best fixed action $\bm{x}^*$ in hindsight is 
\begin{equation}
	\bm{x}^*: = \argmin_{\bm{x} \in {\cal X}} \sum_{t=1}^T f_t(\bm{x}).
\end{equation}
In both MAB and BCO settings, an online algorithm is desirable when its regret is sublinear w.r.t. the time horizon $T$, i.e., $\text{Reg}^{\rm MAB}_T= o(T)$ and $\text{Reg}^{\rm BCO}_T= o(T)$\cite{hazan2016,bubeck2012}.


\subsection{Delayed MAB and BCO}

\textbf{MAB with unknown delays.} In \textit{delayed} MAB, the learner still chooses an arm $ a_t \sim \bm{p}_t$ at the beginning of slot $t$. However, the loss $l_t(a_t)$ is observed after $d_t$ slots, namely, at the end of slot $t+d_t$, where delay $d_t\geq 0$ can vary from slot to slot. In this paper, we assume that $\{d_t\}_{t=1}^T$ can be chosen adversarially by nature. Let $l_{s|t}(a_{s|t})$ denote the loss incurred by the selected arm $a_s$ in slot $s$ but observed at $t$, i.e., the learner receives the losses collected in ${\cal L}_t = \big\{ l_{s|t}(a_{s|t}), ~s:\! s\!+\!d_s \! = \! t\big\}$ at the end of slot $t$. Note that it is possible to have ${\cal L}_t \!=\! \emptyset$ in certain slots. And the order of feedback can be arbitrary, meaning it is possible to heve $t_1 + d_{t_1} \geq t_2 + d_{t_2}$ when $t_1\leq t_2$. In contrast to \cite{joulani2013} however, we consider the case where the delay $d_t$ is not accessible, i.e., the learner just observes the value of $ l_{s|t}(a_{s|t})$, but not $s$. The learner's goal is to select $\big\{\bm{p}_t\big\}_{t=1}^T$ ``on-the-fly'' to minimize the regret defined in \eqref{eq.reg}. Note that in the presence of delays, the information to decide $\bm{p}_t$ is even less compared with the standard MAB. Specifically, the available information for the learner to decide $\bm{p}_t$ is collected in the set ${\cal L}_{1:t-1} := \bigcup_{s=1}^{t-1} {\cal L}_s$; see also Fig. \ref{fig.time_slot_sturcture} (a2).

For simplicity, we assume that all feedback information is received at the end of slot $T$. This assumption does not lose generality since the feedback arriving at the end of slot $T$ cannot aid the arm selection, hence the final performance of the learner will not be affected.

\textbf{BCO with unknown delays.} For delayed BCO, the learner still chooses $\bm{x}_t$ to play while querying $\{\bm{x}_{t,k}\}_{k=1}^{M-1}$ at the beginning of slot $t$. However, the loss as well as the querying responses $\big{\{}f_t(\bm{x}_{t,k}) \big{\}}_{k=0}^{M-1}$ are observed at the end of slot $t+d_t$. Similarly, let $f_{s|t}(\bm{x}_{s|t})$ denote the loss incurred in slot $s$ but observed at $t$, and the feedback set at the end of slot $t$ is ${\cal L}_t = \big\{ \{f_{s|t}(\bm{x}_{s|t,k})\}_{k=0}^{M-1}, s:\! s\!+\!d_s \! = \! t \big\}$. To find a desirable $\bm{x}_t$, the learner relies on history ${\cal L}_{1:t-1} := \bigcup_{s=1}^{t-1} {\cal L}_s$, with the goal of minimizing the regret in \eqref{eq.reg2}. 

\textbf{Why existing algorithms fail with unknown delays?} 
The algorithms for standard (non-delayed) MAB, such as EXP3, cannot be applied to delayed MAB with unknown delays. Recall that in settings without delay, to deal with the partially observed $\bm{l}_t$, EXP3 relies on an importance sampling type of loss estimates given by \cite{auer2002a}
\begin{equation}\label{est.exp3}
	\hat{l}_t(k) = \frac{l_t(a_t) \mathds{1}(a_t =k)}{p_t(k)}\:, ~~~~ k=1,\ldots,K\;.
\end{equation}
The denominator as well as the indicator function in \eqref{est.exp3} ensure unbiasedness of $\hat{l}_t(k)$. Leveraging the estimated loss, the distribution $\bm{p}_{t+1}$ is obtained by
\begin{equation}
	p_{t+1}(k) = \frac{p_t(k) \exp(-\eta \hat{l}_t(k))}{\sum_{j=1}^K p_t(j) \exp(-\eta \hat{l}_t(j))}, ~\forall k
\end{equation}
where $\eta$ is the learning rate. Consider now that the loss $l_{s|t}(a_{s|t})$ with delay $d_s$ is observed at $t=s+d_s$. To recover the unbiased estimator $\hat{l}_{s|t}(k)$ in \eqref{est.exp3}, $p_s(k)$ must be known. However, since $d_s$ is not revealed, even if the learner can store the previous probability distributions, it is not clear how to attain the loss estimator.

Knowing the delay is also instrumental when it comes to the gradient estimator in BCO as well. For non-delayed single-point feedback BCO \cite{flaxman2005}, e.g., $M\!=\!1$, since only one value of the loss instead of the full gradient is observed per slot, the idea is to draw $\bm{u}_t$ uniformly from the surface of a unit ball in $\mathbb{R}^K$, and form the gradient estimate as 
\begin{equation}\label{eq.FKM_grad}
	\bm{g}_t = \frac{K}{\delta}f_t(\bm{x}_t + \delta \bm{u}_t)\bm{u}_t
\end{equation}
where $\delta$ is a small constant. The next action is obtained using a standard online (projected) gradient descent iteration leveraging the estimated gradient, that is
\begin{equation}
	\bm{x}_{t+1} = \Pi_{{\cal X}_\delta} [\bm{x}_t - \eta  \bm{g}_t]
\end{equation}
where ${\cal X}_\delta$ is the shrunk feasibility set to ensure $\bm{x}_{t}+ \bm{u}_{t}$ is feasible. While $\bm{g}_t$ serves as a nearly unbiased estimator of $\nabla f_t (\bm{x}_t)$, the unknown delay brings mismatch between the feedback $f_{s|t}( \bm{x}_{s|t} + \delta \bm{u}_{s|t})$ and $\bm{g}_{s|t}$. 
Specifically, given the feedback $f_{s|t}(\bm{x}_{s|t} + \delta \bm{u}_{s|t})$, since $d_s$ is unknown, the learner does not know $\bm{u}_{s|t}$ to obtain $\bm{g}_{s|t}$ in \eqref{eq.FKM_grad}. Similar arguments also hold for BCO with multi-point feedback.

Therefore, performing delayed bandit learning with unknown delays is challenging, and has not been explored.


\section{DEXP3 for Delayed MAB}\label{sec.DEP3}
We start with the non-stochastic MAB setup that is randomized in nature because an arm $a_t$ is chosen randomly per slot according to a $K\times 1$ probability mass vector $\bm{p}_t$.  In this section, we show that for the MAB problem, so long as the (unknown) delay is bounded, based only on a single-point feedback, the randomized algorithm that we term Delayed EXP3 (DEXP3) can cope with unknown delays in MAB through a \textit{biased} loss estimator, and is guaranteed to attain a desirable regret.

Recall that the feedback at slot $t$ includes losses incurred at slots $s_n, n = 1,2, \ldots, |{\cal L}_t|$, where ${\cal L}_t := \big\{ l_{s_n|t}(a_{s_n|t}): \forall s_n \!=\! t\!-\!d_{s_n} \big\}$. 
Once ${\cal L}_t$ is revealed, the learner estimates $\bm{l}_{s_n|t}$ by scaling the observed loss according to $\bm{p}_t$ at the current slot. 
For each $l_{s_n|t}(a_{s_n|t}) \in {\cal L}_t$, the estimator of the loss vector $\bm{l}_{s_n|t}$ is 
\begin{equation}\label{eq.estloss}
	\hat{l}_{s_n|t}(k) = \frac{l_{s_n|t}(k) \mathds{1} \big(a_{s_n|t} = k \big)}{p_{t}(k)},~ \forall k. 
\end{equation}

It is worth mentioning that the index $s_n$ in \eqref{eq.estloss} is only used for analysis while during the implementation, there is no need to know $s_n$.
In contrast to EXP3 \cite{auer2002a} and its variant for delayed MAB with known delays \cite{joulani2013,cesa2016}, our estimator for $l_{s_n|t}(k)$ in \eqref{eq.estloss} turns out to be \textit{biased} since $a_{s_n|t}$ is chosen according to $\bm{p}_{s_n}$ and not according to $\bm{p}_t$, that is
\begin{equation}\label{eq.expectedloss}
	\mathbb{E}_{a_{s_n|t}} \big[ \hat{{l}}_{s_n|t} (k) \big] = \frac{l_{s_n|t}(k) p_{s_n}(k)}{p_{t}(k)}\neq l_{s_n|t}(k).
\end{equation}

\begin{algorithm}[t]
	\caption{\textbf{DEXP3}}\label{algo2}
	\begin{algorithmic}[1]
		\State \textbf{Initialize:}
		$\bm{p}_1 (k)= 1/K, \forall k$.
		\For {$t=1,2\dots,T$}
		\State Select an arm $a_t \sim \bm{p}_t$.
		\State Observe feedback collected in set ${\cal L}_t $. 
		\For {$n=1, 2, \ldots,|{\cal L}_t|$}
			\State  Estimate $\hat{\bm{l}}_{s_n|t}$ via \eqref{eq.estloss} if $l_{s_n|t}(a_{s_n|t}) \in {\cal L}_t$. 
			\State Update $\bm{p}_t^n $ via \eqref{eq.tildew1} - \eqref{eq.tildep1}.
			\EndFor
		\State Obtain $\bm{p}_{t+1}$ via \eqref{eq.p_t+1}. 
		\EndFor
	\end{algorithmic}
\end{algorithm}

\begin{figure}[t]
	\centering
	\includegraphics[height=0.16\textwidth]{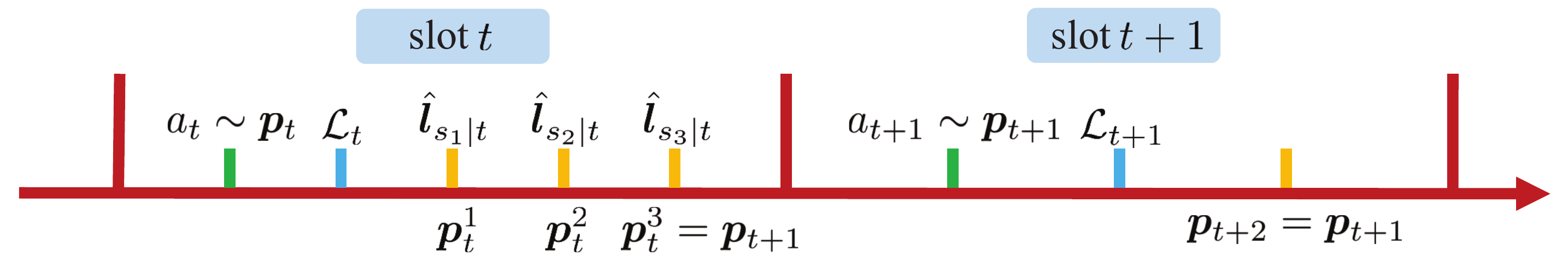}
	\vspace{-0.2cm}
	\caption{An example of DEXP3 with ${\cal L}_t$ including the losses incurred in slot $s_1$, $s_2$, and $s_3$ and ${\cal L}_{t+1} = \emptyset$.}
	\label{fig.DEXP3slot}
\end{figure}

Since ${\cal L}_t$ may contain multiple rounds of feedback, leveraging each  $\hat{\bm{l}}_{s_n|t}$, the learner must update $\big|{\cal L}_t \big|$ times to obtain $\bm{p}_{t+1}$. Intuitively, to upper bound the bias of \eqref{eq.expectedloss}, an upper bound on $p_{s_n}(k)/p_t(k)$ is required, which amounts to a lower bound on $p_t(k)$. On the other hand however, the lower bound of $p_t(k)$ cannot be too large to avoid incurring extra regret. Different from 
EXP3, our DEXP3 ensures a lower bound on $p_t(k)$ by introducing an intermediate weight vector $\tilde{\bm{w}}_t$ to evaluate the historical performance of each arm. Let $n$ denote the index of the inner-loop update at slot $t$ starting from $\bm{p}_t^{0}:=\bm{p}_t$.
For each ${l}_{s_n|t}(a_{s_n|t}) \in {\cal L}_t$, the learner first updates $\tilde{\bm{w}}_t^n$ by using the estimated loss $\hat{\bm{l}}_{s_n|t}$ as
\begin{equation}\label{eq.tildew1}
	\tilde{w}_{t}^n (k) = p_t^{n-1}(k) \exp \Big(-\eta \min\big{\{} \delta_1,\hat{l}_{s_n|t} (k)\big{\}}\Big),~ \forall k
\end{equation}
where $\eta$ is the learning rate, and $\delta_1$ serves as an upper bound of $\hat{{l}}_{s_n|t}(k)$ to control the bias of $\hat{l}_{s_n|t}(k)$. However, to confine the extra regret incurred by introducing $\delta_1$, a carefully-selected $\delta_1$ should ensure that the probability of having $\hat{l}_{s_n|t}(k)$ larger than $\delta_1$ is small enough. Then the learner finds $\bm{w}_t^n$ by a trimmed normalization as
\begin{equation}\label{eq.w1}
	w_t^n (k) = \max \bigg{\{} \frac{\tilde{w}_t^n (k)}{\sum_{j=1}^K \tilde{w}_t^n (j)}, \frac{\delta_2}{K}\bigg{\}},~ \forall k.
\end{equation}
Update \eqref{eq.w1} ensures that $w_t^n (k)$ is lower bounded by $\delta_2 / K$. Finally, the learner normalizes $\bm{w}_t^n$ to obtain $\bm{p}_t^n$ as
\begin{equation}\label{eq.tildep1}
	p_t^n (k) = \frac{w_t^n (k)}{\sum_{j=1}^K w_t^n (j)},~\forall k.
\end{equation}
It can be shown that $p_t^n(k)$ is lower bounded by $p_t^n(k) \geq \frac{\delta_2}{K(1+\delta_2)}$ [cf. \eqref{eq.ineq4} in supplementary material]. 
After all the elements of in ${\cal L}_t$ have been used, the learner finds $\bm{p}_{t+1}$ via
\begin{equation}\label{eq.p_t+1}
	\bm{p}_{t+1} = \bm{p}_t^{|{\cal L}_t|}.
\end{equation}
Furthermore, if ${\cal L}_t = \emptyset$, the learner directly reuses the previous distribution, i.e., $\bm{p}_{t+1} = \bm{p}_t$, and chooses an arm accordingly. In a nutshell, DEXP3 is summarized in Alg. \ref{algo2}. As it will be shown in Sec. \ref{sec.pfDEXP3}, if the delay $d_t$ is bounded by a constant $\bar{d}$, DEXP3 can guarantee a regret of ${\cal O} \big( \sqrt{K\bar{d}(T+D)}\big)$,  where $D = \sum_{t=1}^T d_t$ is the overall delay.

\begin{remark}
	The recent composite loss wrapper algorithm (abbreviated as CLW) in \cite{cesa2018} can be also applied to the delayed MAB problem. However, CLW is designed for a more general setting with composite and anonymous feedback, and its \textit{efficiency} drops when the previous action $a_{s|t}$ is known. The main differences between DEXP3 and CLW are: i) the loss estimators are different; and, ii) DEXP3 updates $\bm{p}_t$ in every slot, while CLW updates occur every other ${\cal O}( 2 \bar{d})$ slots (thus requiring a larger learning rate). As it will be corroborated by simulations, DEXP3 outperforms CLW in the considered setting.
\end{remark}

\section{DBGD for Delayed BCO}

\begin{algorithm}[t]
	\caption{\textbf{DBGD}}\label{algo1}
	\begin{algorithmic}[1]
		\State \textbf{Initialize:}
		$\bm{x}_1 = \bm{0}$.
		\For {$t=1,2\dots,T$}
		\State Play $\bm{x}_t$, also query $\bm{x}_t+\delta \bm{e}_k, k=1,\ldots,K$.
		\State Observe feedback collected in set ${\cal L}_t $. 
		\If {${\cal L}_t = \emptyset$}
			set $\bm{x}_{t+1} = \bm{x}_t$.
		\Else ~estimate gradient $\bm{g}_{s_n|t}$ via \eqref{g_estimate} if $s_n \!+\! d_s \!=\! t$. 
		\State ~ Update $\bm{x}_{t+1}$ via \eqref{eq.xupdate}.
		\EndIf
		\EndFor
	\end{algorithmic}
\end{algorithm}

In this section, we develop an algorithm that we term Delayed Bandit Gradient Descent (DBGD) based on a \textit{deterministic} approximant of the loss obtained using $M=K+1$ rounds of feedback. DBGD enjoys regret of ${\cal O}\big( \sqrt{T+D}\big)$ for BCO problems even when the delays are unknown.  In practice, $(K+1)$-point feedback can be obtained i) when it is possible to evaluate the loss function easily; and ii) when the slot duration is long, meaning that the algorithm has enough time to query multiple points from the oracle \cite{thune2018}.

The intuition behind our deterministic approximation originates from the gradient definition~\cite{agarwal2010}. Consider for example $\bm{x}\in\mathds{R}^2$, and the gradient $\nabla f(\bm{x}) = [\nabla_1, \nabla_2]^\top$, where 
\begin{equation}\label{eq.2dexample}
 \nabla_{1} \!=\! \lim_{\delta \rightarrow 0} \! \frac{f(\bm{x}\!+\!\delta\bm{e}_1) \!-\! f(\bm{x})}{\delta}; \nabla_{2}\! =\! \lim_{\delta \rightarrow 0} \!\frac{f(\bm{x}\!+\!\delta\bm{e}_2) \!-\! f(\bm{x})}{\delta}.
\end{equation}
Similarly, for a $K$-dimensional $\bm{x}$, if $K+1$ rounds of feedback are available, the gradient can be approximated as
\begin{equation}\label{eq.estgrad}
	\bm{g}_t = \frac{1}{\delta} \sum_{k=1}^K \big( f_t(\bm{x}_t + \delta \bm{e}_k) - f_t(\bm{x}_t) \big) \bm{e}_k
\end{equation}
where $\bm{e}_k : = [0,\ldots,1,\ldots,0]^\top$ denotes the unit vector with $k$-th entry equal $1$. Intuitively, a smaller $\delta$ improves the approximation accuracy. When $f_t$ is further assumed to be linear, $\bm{g}_t$ in  \eqref{eq.estgrad} is unbiased. In this case, the gradient of $f_t$ can be recovered exactly, and thus the setup boils down to a delayed one with full information. However, if $f_t(\cdot)$ is generally convex, $\bm{g}_t$ in \eqref{eq.estgrad} is \textit{biased}.

\begin{figure}[t]
	\centering
	\includegraphics[height=0.16\textwidth]{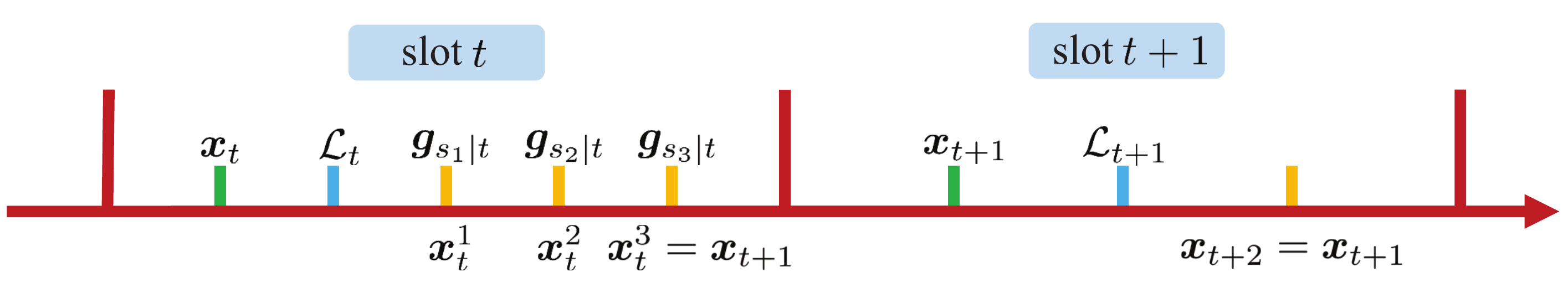}
	\vspace{-0.2cm}
	\caption{An example of DBGD with ${\cal L}_t$ including the losses incurred in slot $s_1$, $s_2$, and $s_3$ and ${\cal L}_{t+1} = \emptyset$.  }
	\label{fig.DBGDslot}
\end{figure}

Leveraging the gradient in \eqref{eq.estgrad}, we are ready to introduce the DBGD algorithm. 
Per slot $t$, the learner plays $\bm{x}_t$ and also queries $f_t(\bm{x}_t + \delta \bm{e}_k), \forall k = \{1,\ldots,K\}$. However, to ensure that $f_{t}(\bm{x}_{t} + \delta \bm{e}_k)$ is feasible, the $\bm{x}_t$ should be confined to the set ${\cal X}_\delta= \{\bm{x}: \frac{\bm{x}}{1-\delta} \in {\cal X} \} $. Note that if $0\leq \delta <1$, ${\cal X}_\delta$ is still convex. Let $n = 1, 2, \ldots, |{\cal L}_t|$ indexing the inner loop update at slot $t$. At the end of slot $t$, the learner receives observations ${\cal L}_t = \big{\{} \{f_{s_n|t}(\bm{x}_{s_n|t}), f_{s_n|t}(\bm{x}_{s_n|t} + \delta \bm{e}_k), k = 1,\ldots, K \}, \forall s_n = t-d_{s_n} \big{\}}$. 
Per received feedback value, the learner approximates the gradient via \eqref{eq.estgrad}; thus, for each $\{ f_{s_n|t}(\bm{x}_{s_n|t}), f_{s_n|t}(\bm{x}_{s_n|t} \!+\! \delta \bm{e}_k), k = 1,\ldots, K \}$, we have
\begin{equation}\label{g_estimate}
\!\bm{g}_{s_n|t} = \frac{1}{\delta} \sum_{k=1}^K \big( f_{s_n|t}(\bm{x}_{s_n|t} + \delta \bm{e}_k) - f_{s_n|t}(\bm{x}_{s_n|t}) \big) \bm{e}_k.\!
\end{equation}
With $\bm{g}_{s_n|t}$ and $\bm{x}_t^{0}:=\bm{x}_t$, the learner will update $\big|{\cal L}_t \big|$ times to obtain $\bm{x}_{t+1}$ by
\begin{subequations}\label{eq.xupdate}
	\begin{equation}
		\bm{x}_t^{n} = \Pi_{{\cal X}_\delta} \big[ \bm{x}_t^{n-1} - \eta \bm{g}_{s_n|t} \big], ~~~n=1,\dots,|{\cal L}_t |
	\end{equation}
	\begin{equation}
		\bm{x}_{t+1} = \bm{x}_t^{|{\cal L}_t |}.
	\end{equation}
\end{subequations}
If no feedback is received at slot $t$, the learner simply sets $\bm{x}_{t+1} = \bm{x}_{t}$. The DBGD is summarized in Algorithm \ref{algo1}.



\section{A Unified Framework for Regret Analysis}\label{sec.analysis}

In this section, we show that both DEXP3 and DBGD can guarantee an ${\cal O}\big(\sqrt{T+D} \big)$ regret. 
Our analysis considerably broadens that in \cite{joulani2016}, which was originally developed for delayed online learning with full-information feedback.

\subsection{Mapping from Real to Virtual Slots}

To analyze the recursion involving consecutive variables ($\bm{p}_t$ and $\bm{p}_{t+1}$ in DEXP3 or $\bm{x}_t$ and $\bm{x}_{t+1}$ in DBGD) is challenging, since different from standard settings, the number of feedback rounds varies over slots. We will bypass this variable feedback using the notion of a ``virtual slot.''

 
Over the real time horizon, there are in total $T$ virtual slots, where the $\tau$th virtual slot is associated with the $\tau$th loss value fed back. 
Recall that the feedback received at the end of slot $t$ is ${\cal L}_t$. With the overall feedback received until the end of slot $t-1$ denoted by $L_{t-1}: = \sum_{v=1}^{t-1}|{\cal L}_v|$, the virtual slot $\tau$ corresponding to the first feedback value received at slot $t$ is $\tau = L_{t-1}+1$. In what follows, we will use MAB as an example to elaborate on this mapping, but the BCO setting can be argued likewise. 


When the multiple rounds of feedback are received over a real slot $t$, DEXP3 updates $| {\cal L}_t |$ times  $\bm{p}_t$ to obtain $\bm{p}_{t+1}$; see \eqref{eq.tildew1} - \eqref{eq.tildep1}. Using the notion of virtual slots, these $| {\cal L}_t |$ updates are performed over $| {\cal L}_t |$ consecutive virtual slots. Taking Fig. \ref{fig.RVlearner} as an example, when $\bm{p}_t^1$ is obtained by using an estimated loss $\hat{\bm{l}}_{s_1|t}$ [cf. \eqref{eq.estloss}] and \eqref{eq.tildew1} - \eqref{eq.tildep1}, this update is mapped to a virtual slot $\tau = L_{t-1} +1$, where $\tilde{\bm{l}}_\tau := \hat{\bm{l}}_{s_1|t}$ is adopted to obtain $\tilde{\bm{p}}_{\tau+1} := \bm{p}_t^1$. Similarly, when $\bm{p}_{t}^2$ is obtained using $\hat{\bm{l}}_{s_2|t}$, the virtual slot yields $\tilde{\bm{p}}_{\tau+2} := \bm{p}_t^2$ via $\tilde{\bm{l}}_{\tau+1}: =\hat{\bm{l}}_{s_2|t}$. 
That is to say, at real slot $t$, for $n = 1,\ldots, |{\cal L}_t|$, each update from $\bm{p}_t^{n-1}$ to $\bm{p}_t^n$ using the estimated loss $\hat{\bm{l}}_{s_n|t}$ is mapped to an ``update'' at the virtual slot $\tau + n-1$, where $\tilde{\bm{l}}_{\tau+n-1}: = \hat{\bm{l}}_{s_n|t}$ is employed to obtain $\tilde{\bm{p}}_{\tau+n}:= \bm{p}_t^n$ from $\tilde{\bm{p}}_{\tau+n-1} =  \bm{p}_t^{n-1}$. According to the real-to-virtual slot mapping, we have $\tilde{\bm{p}}_{\tau + |{\cal L}_t|} = \bm{p}_{t+1}$; see also Fig. \ref{fig.RVlearner} for two examples. As we will show later, it is convenient to analyze the recursion between two consecutive $\tilde{\bm{p}}_\tau$ and $\tilde{\bm{p}}_{\tau+1}$, which is the key for the ensuing regret analysis.

\begin{figure}[t]
	\centering
	\includegraphics[height=0.4\textwidth]{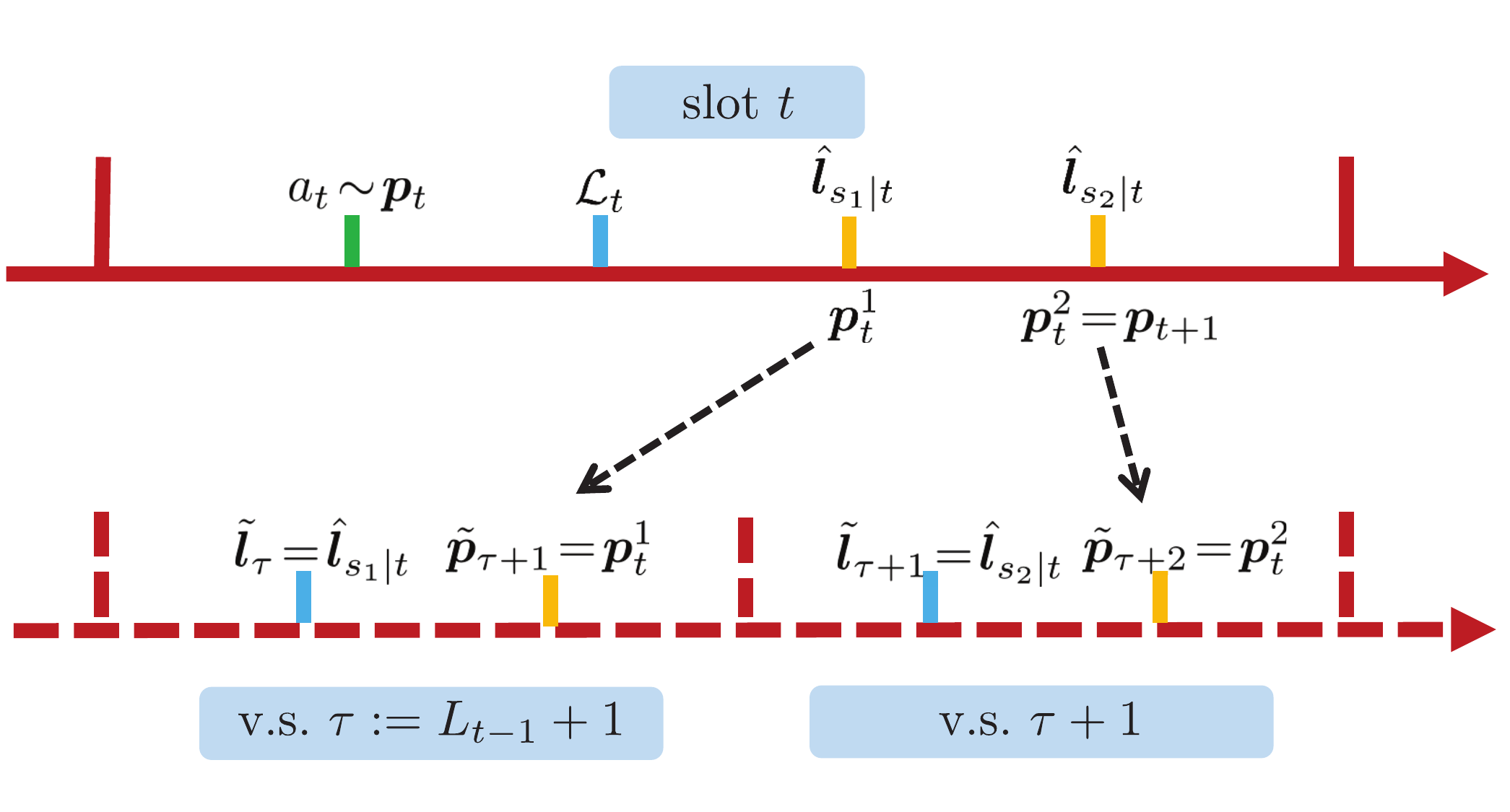}
	\vspace{-0.2cm} 
	\caption{An example of mapping from real slots (solid line) and virtual slots (dotted line). At the end of slot $t$, the feedback is ${\cal L}_t = \{ l_{s_1|t} (a_{s_1|t}), l_{s_2|t} (a_{s_2|t})\}$.
``v.s.'' stands for virtual slot.
}
	\label{fig.RVlearner}
	\vspace{-0.2cm}
\end{figure}

With regard to DBGD, since multiple feedback rounds are possible per real slot $t$, we again map the $| {\cal L}_t |$ updates at a real slot [cf. \eqref{eq.xupdate}] to $| {\cal L}_t |$ virtual slots. The mapping is exactly the same as that in DEXP3, that is, per virtual slot $\tau$, $\tilde{\bm{g}}_\tau$ is used to obtain $\tilde{\bm{x}}_{\tau+1}$.

From this real-to-virtual mapping vantage point, DEXP3 and DBGD can be viewed as (inexact) EXP3 and BGD running on the virtual time horizon with only one feedback value per virtual slot. That is to say, instead of analyzing regret on the real time horizon, which can involve multiple feedback rounds, we can alternatively turn to the virtual slot, where there is only one ``update'' per slot. 

\subsection{Regret Analysis of DEXP3}\label{sec.pfDEXP3}

Now we turn to the analyze the regret for DEXP3. The analysis builds on the following assumptions.
\begin{assumption}\label{as.1}
	The losses satisfy $\max_{t,k} l_t(k) \leq 1$.
\end{assumption}
\begin{assumption}\label{as.2}
	The delay $d_t$ is bounded, i.e., $\max_t d_t \leq \bar{d}$.
\end{assumption}
Assumption \ref{as.1} requires the loss function to be upper bounded, which is common in MAB; see also \cite{hazan2016,auer2002a,bubeck2012}.
Assumption \ref{as.2} asks for the delay to be bounded that also appears in previous analyses for the delayed online learning setup \cite{joulani2013,cesa2018,cesa2016}.
It is also assumed that $\bar{d}>0$, since otherwise DEXP3 boils down to EXP3 automatically.
Let us first consider the changes on $\tilde{p}_{\tau}(k)$ after one ``update'' in the virtual slot. 
\begin{lemma}\label{lemma.small/big} 
	If the parameters are properly selected such that $1-\delta_2 - \eta \delta_1 \geq 0$, the following inequality holds 
	\begin{equation}\label{eq.small/big.1}
		\frac{\tilde{p}_{\tau -1}(k)}{\tilde{p}_{\tau}(k)} \leq \frac{1}{1-\delta_2 - \eta \delta_1},~~~\forall k, \tau.
	\end{equation}
\end{lemma}
\begin{proof}
	See Sec. B.1 of the supplementary document.
\end{proof}

\begin{lemma}\label{lemma.big/small}
	If the parameters are properly selected such that $1 - \eta \delta_1 \geq 0$, the following inequality holds 
	\begin{equation}\label{eq.big/small.1}
		\frac{\tilde{p}_{\tau}(k)}{\tilde{p}_{\tau-1}(k)} \leq \max\left\{1+\delta_2, \frac{1}{1-\eta \delta_1}\right\},~~~\forall k, \tau.
	\end{equation}
\end{lemma}
\begin{proof}
	See Sec. B.2 of the supplementary document.
\end{proof}
Lemmas \ref{lemma.small/big} and \ref{lemma.big/small} assert that both $\tilde{p}_{\tau-1}(k)/\tilde{p}_\tau(k)$ and $\tilde{p}_\tau(k)/\tilde{p}_{\tau-1}(k)$ are bounded deterministically, that is, regardless of the arm selection and observed loss. These bounds are the critical for deriving the regret. 

To bound the regret, the final cornerstone is the ``regret'' in virtual slots, specified in the following lemma.

\begin{lemma}\label{lemma.innerreg}
	For a given sequence of $\{ \tilde{\bm{l}}_\tau \}_{\tau = 1}^T$, the following relation follows
	\begin{align}
		& ~~~~~  \sum_{\tau=1}^T \big(\tilde{\bm{p}}_\tau - \bm{p} \big)^\top \min \big{\{} \tilde{\bm{l}}_\tau, \delta_1 \cdot \bm{1} \big{\}} \leq \!\frac{T \ln (1+\delta_2)+ \ln K}{\eta} + \frac{\eta}{2} \sum_{\tau=1}^T \sum_{k=1}^K \tilde{p}_\tau (k) \big[ \tilde{l}_\tau (k) \big]^2
	\end{align}
	where $ \bm{1}$ is a $K\times 1$ vector of all ones, and $\bm{p} \in \Delta_K$.
\end{lemma}
\begin{proof}
	See Sec. B.3 of the supplementary document.
\end{proof}

Leveraging Lemma \ref{lemma.innerreg}, the regret of DEXP3 follows.

\begin{theorem}\label{theo.reg}
Supposing Assumptions \ref{as.1} and \ref{as.2} hold, defining the overall delay $D := \sum_{t=1}^T d_t$, and choosing $\delta_2 = \frac{1}{T+D}$,  $\eta = {\cal O} \Big( \sqrt{ \frac{1+\ln K}{\bar{d}K(T+D)} } \Big)$, and $\delta_1 = \frac{1}{2\eta\bar{d}} - \frac{\delta_2}{\eta}$, DEXP3 guarantees that the ${\rm Reg}_T^{\rm MAB}$ in \eqref{eq.reg} satisfies
	\begin{equation}
		{\rm Reg}_T^{\rm MAB} = {\cal O} \big( \sqrt{ (T+D) \bar{d}K (1\!+\!\ln K)} \big).
	\end{equation}
\end{theorem}
\begin{proof}
	See Sec. B.4 of the supplementary document.
\end{proof}

Theorem \ref{theo.reg} indicates that DEXP3 tends to perform well when $D$ is small, for example when $D = o(T)$. This can happen when the delay is sparse, that is, most of $d_t = 0$. 


\subsection{Regret analysis of DBGD}

Our analysis builds on the following assumptions.

\begin{assumption}\label{as.3}
For any $t$, the loss function $f_t$ is convex.
\end{assumption}
\begin{assumption}\label{as.4}
	For any $t$, $f_t$ is $L$-Lipschitz and $\beta$-smooth.
\end{assumption}
\begin{assumption}\label{as.5}
	The feasible set contains $\epsilon {\cal B}$, where ${\cal B}$ is the unit ball, and $\epsilon > 0$ is a predefined parameter. The diameter of ${\cal X}$ is $R$; that is, $\max_{\bm{x},\bm{y} \in {\cal X}} \|\bm{x}-\bm{y} \| = R$.
\end{assumption}
 
Assumptions \ref{as.3} - \ref{as.5} are common in online learning \cite{hazan2016}.
Assumption \ref{as.4} requires that $f_t(\cdot)$ is $L$-Lipschitz and $\beta$-smooth, which is needed to bound the bias of the estimator $\bm{g}_{s|t}$ \cite{agarwal2010}. Assumption \ref{as.5} is also typical in BCO \cite{hazan2016,agarwal2010,flaxman2005,duchi2015}. In addition, the counterpart of Assumption \ref{as.1} in BCO can readily follow from Assumptions \ref{as.4} and \ref{as.5}.

To start, the quality of the gradient estimator $\bm{g}_{s|t}$ is first evaluated. As stated, when $f_t(\cdot)$ is not linear, the estimator $\bm{g}_{s|t}$ is biased, but its bias is bounded.
\begin{lemma}\label{lemma.Gbound}
If Assumption \ref{as.4} holds, then for every $f_{s|t}(\bm{x}_{s|t})$, the corresponding estimator \eqref{eq.estgrad} satisfies
	\begin{equation}
		\| \bm{g}_{s|t}\| \!\leq\! \sqrt{K}L, ~\text{and}~~ \| \bm{g}_{s|t} - \nabla f_{s|t}(\bm{x}_{s|t}) \| \!\leq\! \frac{\beta \delta}{2}\sqrt{K}.
	\end{equation}
\end{lemma}
\begin{proof}
	See Sec. C.1 of the supplementary document. 
\end{proof}

Lemma \ref{lemma.Gbound} suggests that with $\delta$ small enough, the bias of $\bm{g}_{s|t}$ will not be too large. Then, the following lemma shows the relation among $\tilde{\bm{g}}_\tau$, $\tilde{\bm{x}}_{\tau}$, and $\tilde{\bm{x}}_{\tau+1}$, in a virtual time slot.
\begin{lemma}\label{lemma.descent}
Under Assumptions \ref{as.4} and \ref{as.5}, the update in a virtual slot $\tau$ guarantees that
for any $\bm{x} \in {\cal X}_\delta$, we have
	\begin{equation}
		\tilde{\bm{g}}_\tau^\top \big( \tilde{\bm{x}}_\tau - \bm{x} \big) \leq \frac{\eta}{2} KL^2 +\frac{ \big{\|} \tilde{\bm{x}}_{\tau} - \bm{x} \big{\|}^2 - \big{\|} \tilde{\bm{x}}_{\tau+1} - \bm{x} \big{\|}^2}{2\eta}. 
	\end{equation}
\end{lemma}
\begin{proof}
	See Sec. C.2 of the supplementary document.
\end{proof}

Lemma \ref{lemma.descent} is the counterpart of the gradient descent estimate in the non-delayed and full-information setting \cite[Theorem 3.1]{hazan2016}, which demonstrates that DBGD is BGD running on the virtual slots. 
Finally, leveraging these results, the regret bound follows next.
\begin{theorem}\label{theo.reg2}
Suppose Assumptions \ref{as.3} - \ref{as.5} hold. Choosing $\delta= {\cal O}\big((T+D)^{-1}\big)$, and $\eta = {\cal O}\big((T+D)^{-1/2}\big)$, the DBGD guarantees that	the regret is bounded, that is
	 \begin{equation}
\!{\rm Reg}_T^{{\rm BCO}}\! =\! \sum_{t=1}^T f_t(\bm{x}_t) - \sum_{t=1}^Tf_t(\bm{x}^*) = {\cal O} \big( \sqrt{T+D} \big)\!
	\end{equation}
	where $D:= \sum_{t=1}^T d_t$ is the overall delay.
\end{theorem}
\begin{proof}
	See Sec. C.3 of the supplementary document.
\end{proof}

For the slightly different regret definition in \cite{agarwal2010}, DBGD achieves the same regret bound.
\begin{corollary}\label{coro.reg}
Upon defining $\bm{x}_{t,0}:= \bm{x}_t$ and $\bm{x}_{t,k}: =\bm{x}_t+\delta \bm{e}_k$, choosing $\delta= {\cal O}\big((T+D)^{-1}\big)$, and $\eta = {\cal O}\big((T+D)^{-1/2}\big)$, the DBGD also guarantees that
	\begin{align}
		\frac{1}{K\!+\!1}\sum_{t=1}^T\sum_{k=0}^{K} \! f_t(\bm{x}_{t,k}) \!-\! \sum_{t=1}^T \! f_t(\bm{x}^*) = {\cal O}\big( \sqrt{T\!+\!D} \big).
	\end{align}
\end{corollary}
\begin{proof}
	See Sec. C.4 of the supplementary document.
\end{proof}

The ${\cal O} \big( \sqrt{T+D} \big)$ regret in Theorem \ref{theo.reg2} and Corollary \ref{coro.reg} recovers the bound of delayed online learning in the full information setup \cite{quanrud2015,langford2009,joulani2016} with only bandit feedback. 

\section{Numerical tests}


\begin{figure*}[t]
\begin{tabular}{cc}
\hspace*{-2ex}
\includegraphics[width=8.5cm]{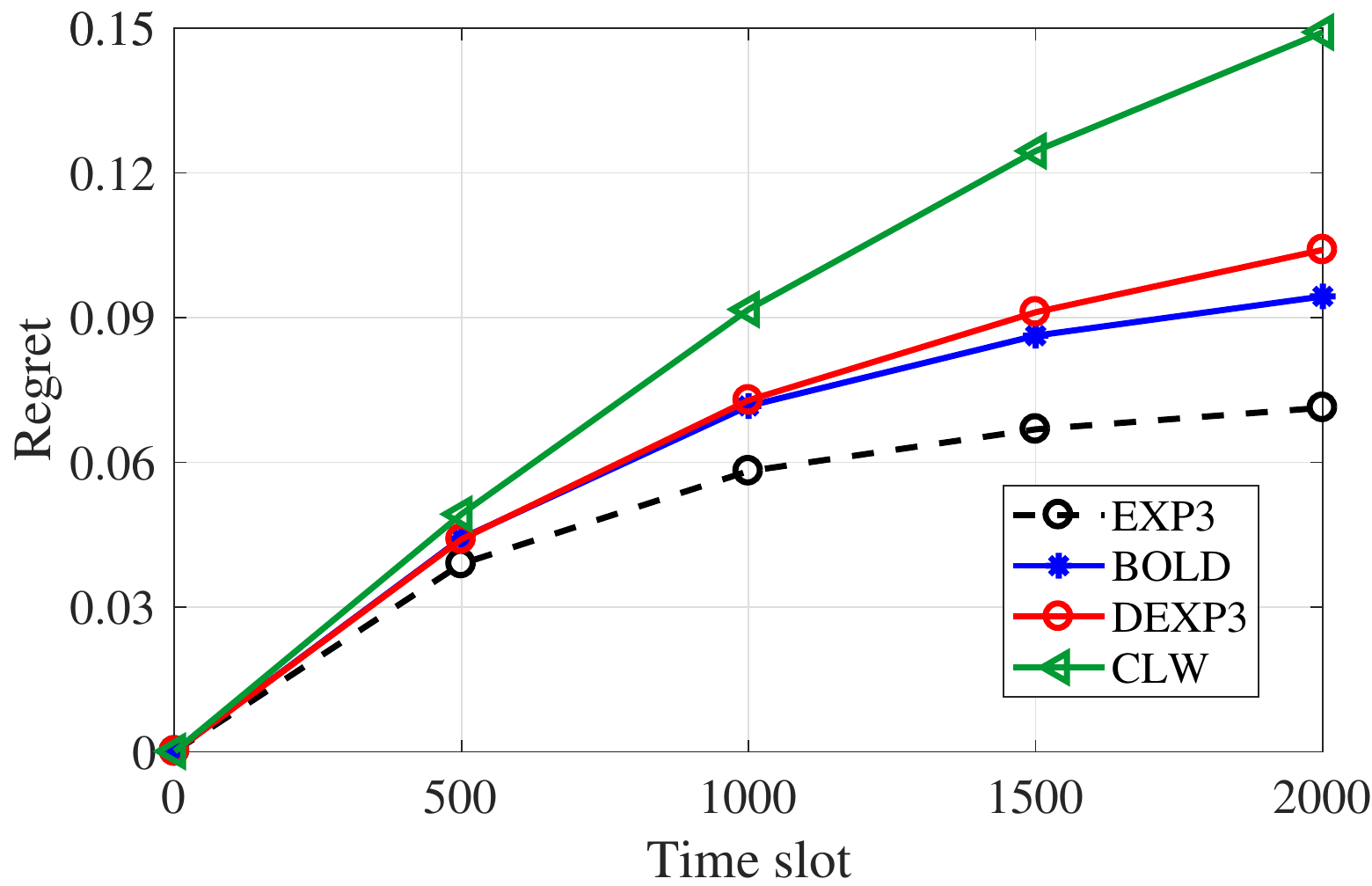}&
\hspace*{-2ex}
\includegraphics[width=8.5cm]{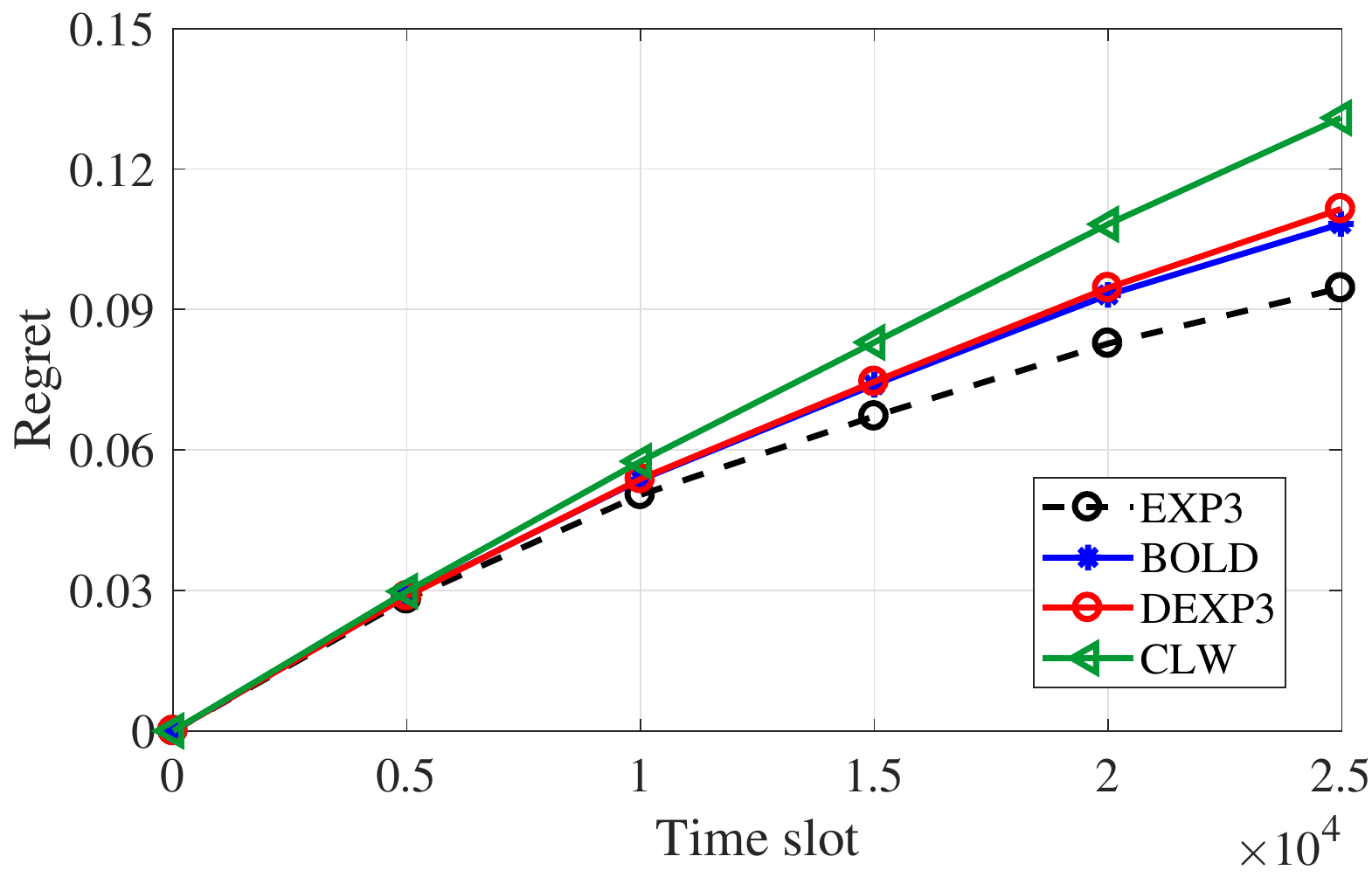}
\\
(a)& (b)
\end{tabular}
\caption{(a) Regret of DEXP3 using synthetic data; (b) Regret of DEXP3 using real data.} \label{fig.mab}
\end{figure*}

%

\begin{figure*}[t]
\begin{tabular}{cc}
\hspace*{-2ex}
\includegraphics[width=8.5cm]{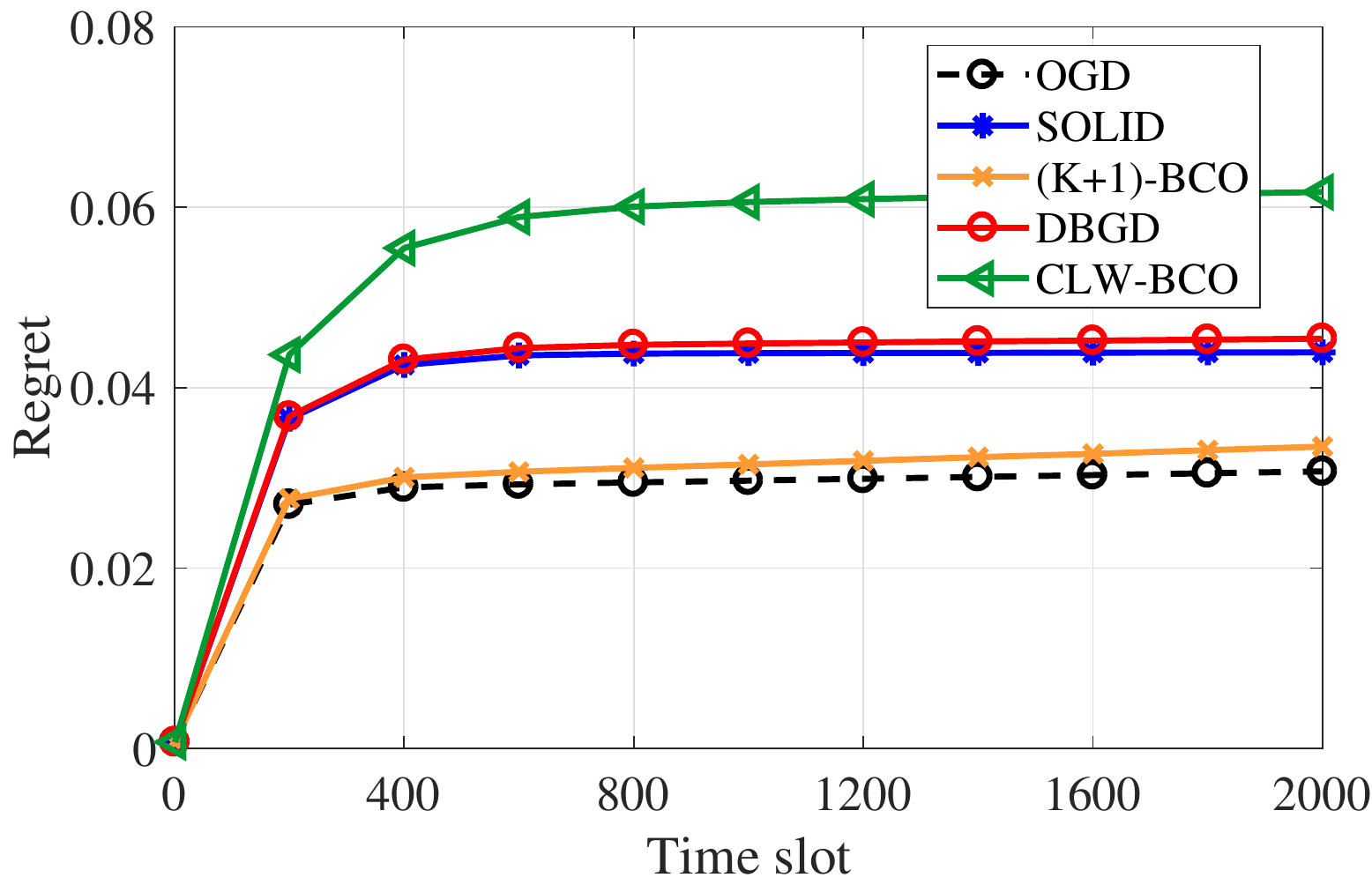}&
\hspace*{-2ex}
\includegraphics[width=8.5cm]{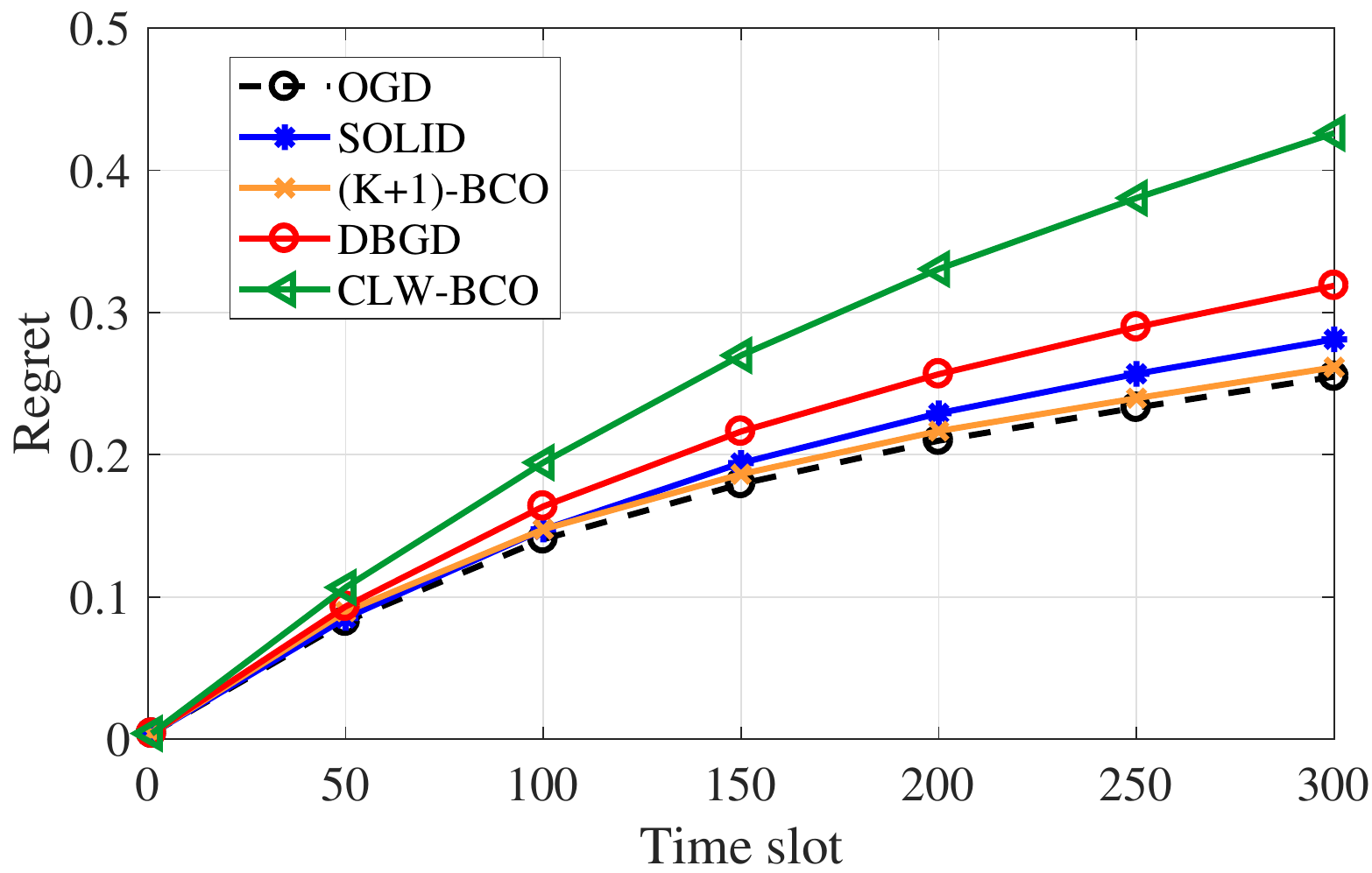}
\\
(a)& (b)
\end{tabular}
\caption{(a) Regret of DBGD using synthetic data; (b) Regret of DBGD using real data.} \label{fig.bco}
\end{figure*}

%
%

In this section, experiments are conducted to corroborate the validity of the novel DEXP3 and DBGD schemes.

In synthetic data tests, we consider $T = 2,000$ slots. Delays are periodically generated with period $1,2,1,0,3,0,2$, with the delay of the last few slots slightly modified to ensure that all feedback arrives at the end of $T = 2,000$, resulting in the overall delay $D = 2,569$.

\paragraph{DEXP3 synthetic tests.}
Consider $K = 5$ arms, and losses generated with a sudden change. Specifically, for $t\in [1,500]$, we have $l_t(k) = 0.4 k |\cos t|$ per arm $k$, while for the rest of the slots, $l_t(k) = 0.2 k |\sin (2t)|$. To benchmark the novel DEXP3, we use: i) the standard EXP3 for non-delayed MAB~\cite{auer2002a};  ii) the BOLD for delayed MAB with known delay~\cite{joulani2013}; and, iii) CLW to deal with the more difficult setting in~\cite{cesa2018}. The instantaneous accumulated regret (normalized by $T$) versus time slots is plotted in Fig. \ref{fig.mab} (a). The gap between BOLD and EXP3, illustrates that even with a known delay, the learner suffers from an extra regret. The small gap between DEXP3 and BOLD further demonstrates the estimator bias [cf. \eqref{eq.estloss}] causing a slightly larger regret, which is the price paid for the unknown delay. 
Compared with CLW, DEXP3 performs significantly better since DEXP3 can leverage more information relative to the non-anonymous feedback used by CLW.


\paragraph{DEXP3 real tests.}
We also tested DEXP3 using the \textit{Jester Online Joke Recommender System} dataset \cite{goldberg2001}, where $T = 24,983$ users rate $K = 100$ different jokes from $0$ (not funny) to $1$ (very funny). The goal is to recommend one joke per slot $t$ to amuse the users. The system performance is evaluated by $\bm{l}_t = (\bm{1} - \text{the score of this joke})$. In this test, we assign a random score within range $[0,1]$ for the missing entries of this dataset. The delay is generated periodically as in the synthetic test, resulting in $D = 32,119$. Similar to the synthetic test, it can be observed that DEXP3 incurs slightly larger regret than BOLD due to the unknown delay, but outperforms the recently developed CLW.

\paragraph{DBGD synthetic tests.}
Consider that $K = 5$, and the feasible set ${\cal X} \in \mathbb{R}^5$ is the unit ball, i.e., ${\cal X}:= \big{ \{} \|\bm{x} \| \leq 1 \big{\}} $. The loss function at slot $t$ is generated as $f_t(\bm{x}) = a_t \|\bm{x} \|^2 + \bm{b}_t^\top \bm{x}$, where $a_t = \cos(3t)+3$, while $b_t(1) = 2 \sin(2t) + 1$,$b_t(2) = \cos(2t) - 2$, $b_t(3) = \sin(2t)$, $b_t(4) = 2\sin(2t) -2$, $b_t(4) = 2\sin(2t) -2$, and $b_t(5) = 2$. To see the influence of the bandit feedback and the unknown delay, we consider the following benchmarks: i) the standard OGD \cite{zinkevich2003} in full-information and non-delayed setting; ii) the $(K+1)$-point feedback BCO \cite{agarwal2010} for non-delayed BCO; iii) the SOLID for delayed full-information OCO~\cite{joulani2016}; and iv) the CLW-BCO with the inner algorithm relying on $(K+1)$-point feedback BCO~\cite{cesa2018}. The instantaneous accumulated regret (normalized by $T$) versus time slots is plotted in Fig. \ref{fig.bco} (a). This test shows that DBGD performs almost as good as SOLID, and the gap between DBGD/SOLID and OGD/$K+1$-BCO is due to the delay. 
The regret of DBGD again significantly outperforms that of CLW-BCO, demonstrating the efficiency of DBGD. 

\paragraph{DBGD real tests.}
To further illustrate the merits of DBGD, we conduct tests dealing with online regression applied to a \textit{yacht hydrodynamics} dataset \cite{dua2017}, which contains $T = 308$ data with $K = 6$ features. Per slot $t$, the regressor $\bm{x}_t \in \mathbb{R}^6$ predicts based on the feature $\bm{w}_t$ before its label $y_t$ is revealed. The loss function for slot $t$ is $f_t(\bm{x}_t) = \frac{1}{2}(y_t - \bm{x}_t^\top \bm{w}_t )^2$. The delay is generated periodically as before, and cumulatively it is $D = 394$. The instantaneous accumulated regret (normalized by $T$) versus time slots is plotted in Fig. \ref{fig.bco} (b). Again, DBGD outperforms CLW-BCO considerably. 
Comparing with the regret performance of DBGD and SOLID, we can safely deduce the influence delay has on bandit feedback. 

\section{Conclusions}
Bandit online learning with unknown delays, including non-stochastic MAB and BCO, was studied in this paper. Different from settings where the experienced delay is known in bandit online learning, the unknown delay prevents a simple gradient estimate that is needed by the iterative algorithm. To address this issue, a biased loss estimator as well as a deterministic one were developed for non-stochastic MAB and BCO. Leveraging the proposed loss estimators, the so-termed DEXP3 and DBGD algorithms were developed. The regret of both DEXP3 and DBGD were established analytically. Numerical tests on synthetic and real datasets confirmed the performance gain of DEXP3 and DBGD relative to state-of-the-art approaches.

\newpage
\bibliographystyle{IEEEtranS}
\bibliography{myabrv,datactr}

\begin{thebibliography}{10}
\providecommand{\url}[1]{#1}
\csname url@samestyle\endcsname
\providecommand{\newblock}{\relax}
\providecommand{\bibinfo}[2]{#2}
\providecommand{\BIBentrySTDinterwordspacing}{\spaceskip=0pt\relax}
\providecommand{\BIBentryALTinterwordstretchfactor}{4}
\providecommand{\BIBentryALTinterwordspacing}{\spaceskip=\fontdimen2\font plus
\BIBentryALTinterwordstretchfactor\fontdimen3\font minus
  \fontdimen4\font\relax}
\providecommand{\BIBforeignlanguage}[2]{{%
\expandafter\ifx\csname l@#1\endcsname\relax
\typeout{** WARNING: IEEEtranS.bst: No hyphenation pattern has been}%
\typeout{** loaded for the language `#1'. Using the pattern for}%
\typeout{** the default language instead.}%
\else
\language=\csname l@#1\endcsname
\fi
#2}}
\providecommand{\BIBdecl}{\relax}
\BIBdecl

\bibitem{agarwal2010}
A.~Agarwal, O.~Dekel, and L.~Xiao, ``Optimal algorithms for online convex
  optimization with multi-point bandit feedback.'' in \emph{Proc. Intl. Conf.
  on Learning Theory}, 2010, pp. 28--40.

\bibitem{agarwal2011}
A.~Agarwal and J.~C. Duchi, ``Distributed delayed stochastic optimization,'' in
  \emph{Proc. Advances in Neural Info. Process. Syst.}, Granada, Spain, 2011,
  pp. 873--881.

\bibitem{auer2002a}
P.~Auer, N.~Cesa-Bianchi, Y.~Freund, and R.~E. Schapire, ``The nonstochastic
  multiarmed bandit problem,'' \emph{SIAM Journal on Computing}, vol.~32,
  no.~1, pp. 48--77, 2002.

\bibitem{awerbuch2004}
B.~Awerbuch and R.~D. Kleinberg, ``Adaptive routing with end-to-end feedback:
  Distributed learning and geometric approaches,'' in \emph{Proc. ACM Symp. on
  Theory of Computing}, Chicago, IL, Jun. 2004, pp. 45--53.

\bibitem{bubeck2012}
S.~Bubeck, N.~Cesa-Bianchi \emph{et~al.}, ``Regret analysis of stochastic and
  nonstochastic multi-armed bandit problems,'' \emph{Found. and
  Trends{\textregistered} in Machine Learning}, vol.~5, no.~1, pp. 1--122,
  2012.

\bibitem{cesa2018}
N.~Cesa-Bianchi, C.~Gentile, and Y.~Mansour, ``Nonstochastic bandits with
  composite anonymous feedback,'' in \emph{Proc. Conf. On Learning Theory},
  Stockholm, Sweden, 2018, pp. 750--773.

\bibitem{cesa2016}
N.~Cesa-Bianchi, C.~Gentile, Y.~Mansour, and A.~Minora, ``Delay and cooperation
  in nonstochastic bandits,'' \emph{J. Machine Learning Res.}, vol.~49, pp.
  605--622, 2016.

\bibitem{chapelle2011}
O.~Chapelle and L.~Li, ``An empirical evaluation of thompson sampling,'' in
  \emph{Proc. Advances in Neural Info. Process. Syst.}, Granada, Spain, 2011,
  pp. 2249--2257.

\bibitem{desautels2014}
T.~Desautels, A.~Krause, and J.~W. Burdick, ``Parallelizing
  exploration-exploitation tradeoffs in gaussian process bandit optimization,''
  \emph{J. Machine Learning Res.}, vol.~15, no.~1, pp. 3873--3923, 2014.

\bibitem{dua2017}
\BIBentryALTinterwordspacing
D.~Dheeru and E.~Karra~Taniskidou, ``{UCI} machine learning repository,'' 2017.
  [Online]. Available: \url{http://archive.ics.uci.edu/ml}
\BIBentrySTDinterwordspacing

\bibitem{duchi2013}
J.~Duchi, M.~I. Jordan, and B.~McMahan, ``Estimation, optimization, and
  parallelism when data is sparse,'' in \emph{Proc. Advances in Neural Info.
  Process. Syst.}, Lake Tahoe, Nevada, 2013, pp. 2832--2840.

\bibitem{duchi2015}
J.~C. Duchi, M.~I. Jordan, M.~J. Wainwright, and A.~Wibisono, ``Optimal rates
  for zero-order convex optimization: The power of two function evaluations,''
  \emph{{IEEE} Trans. Inform. Theory}, vol.~61, no.~5, pp. 2788--2806, 2015.

\bibitem{flaxman2005}
A.~D. Flaxman, A.~T. Kalai, and H.~B. McMahan, ``Online convex optimization in
  the bandit setting: gradient descent without a gradient,'' in \emph{Proc. of
  ACM-SIAM symposium on Discrete algorithms}, Vancouver, Canada, pp. 385--394.

\bibitem{goldberg2001}
\BIBentryALTinterwordspacing
K.~Goldberg, T.~Roeder, D.~Gupta, and C.~Perkins, ``Eigentaste: A constant time
  collaborative filtering algorithm,'' \emph{Information Retrieval}, vol.~4,
  no.~2, pp. 133--151, 2001. [Online]. Available:
  \url{http://eigentaste.berkeley.edu/dataset/}
\BIBentrySTDinterwordspacing

\bibitem{hazan2016}
E.~Hazan, ``Introduction to online convex optimization,'' \emph{Found. and
  Trends{\textregistered} in Optimization}, vol.~2, no. 3-4, pp. 157--325,
  2016.

\bibitem{joulani2013}
P.~Joulani, A.~Gyorgy, and C.~Szepesv{\'a}ri, ``Online learning under delayed
  feedback,'' in \emph{Proc. Intl. Conf. Machine Learning}, Atlanta, 2013, pp.
  1453--1461.

\bibitem{joulani2016}
P.~Joulani, A.~Gy{\'o}rgy, and C.~Szepesv{\'a}ri, ``Delay-tolerant online
  convex optimization: Unified analysis and adaptive-gradient algorithms,'' in
  \emph{Proc. of AAAI Conf. on Artificial Intelligence}, vol.~16, Phoenix,
  Arizona, 2016, pp. 1744--1750.

\bibitem{langford2009}
J.~Langford, A.~J. Smola, and M.~Zinkevich, ``Slow learners are fast,''
  \emph{Proc. Advances in Neural Info. Process. Syst.}, pp. 2331--2339, 2009.

\bibitem{li2010}
L.~Li, W.~Chu, J.~Langford, and R.~E. Schapire, ``A contextual-bandit approach
  to personalized news article recommendation,'' in \emph{Proc. of the 19th
  Intl. Conf. on World Wide Web}.\hskip 1em plus 0.5em minus 0.4em\relax
  Rayleigh, NC: ACM, 2010, pp. 661--670.

\bibitem{mcmahan2014}
B.~McMahan and M.~Streeter, ``Delay-tolerant algorithms for asynchronous
  distributed online learning,'' in \emph{Proc. Advances in Neural Info.
  Process. Syst.}, Montreal, Canada, 2014, pp. 2915--2923.

\bibitem{mcmahan2017}
H.~B. McMahan, E.~Moore, D.~Ramage, S.~Hampson \emph{et~al.},
  ``Communication-efficient learning of deep networks from decentralized
  data,'' in \emph{Proc. Intl. Conf. on Artificial Intelligence and
  Statistics}, Fort Lauderdale, Florida, 2017, pp. 273--1282.

\bibitem{nesterov1998}
Y.~Nesterov, \emph{Introductory lectures on convex optimization: A basic
  course}.\hskip 1em plus 0.5em minus 0.4em\relax Springer Science \& Business
  Media, 2013, vol.~87.

\bibitem{neu2010}
G.~Neu, A.~Antos, A.~Gy{\'o}rgy, and C.~Szepesv{\'a}ri, ``Online markov
  decision processes under bandit feedback,'' in \emph{Proc. Advances in Neural
  Info. Process. Syst.}, Vancouver, Canada, 2010, pp. 1804--1812.

\bibitem{pike2017}
C.~Pike-Burke, S.~Agrawal, C.~Szepesvari, and S.~Grunewalder, ``Bandits with
  delayed anonymous feedback,'' \emph{arXiv preprint arXiv:1709.06853}, 2017.

\bibitem{quanrud2015}
K.~Quanrud and D.~Khashabi, ``Online learning with adversarial delays,'' in
  \emph{Proc. Advances in Neural Info. Process. Syst.}, Montreal, Canada, 2015,
  pp. 1270--1278.

\bibitem{shamir2017}
O.~Shamir and L.~Szlak, ``Online learning with local permutations and delayed
  feedback,'' in \emph{Proc. Intl. Conf. Machine Learning}, Sydney, Australia,
  2017, pp. 3086--3094.

\bibitem{thune2018}
T.~S. Thune and Y.~Seldin, ``Adaptation to easy data in prediction with limited
  advice,'' \emph{arXiv preprint arXiv:1807.00636}, 2018.

\bibitem{vernade2017}
C.~Vernade, O.~Capp{\'e}, and V.~Perchet, ``Stochastic bandit models for
  delayed conversions,'' in \emph{Proc. Conf. on Uncertainty in Artificial
  Intelligence}, Sydney, Australia, 2017.

\bibitem{weinberger2002}
M.~J. Weinberger and E.~Ordentlich, ``On delayed prediction of individual
  sequences,'' \emph{{IEEE} Trans. Inform. Theory}, vol.~48, no.~7, pp.
  1959--1976, 2002.

\bibitem{zinkevich2003}
M.~Zinkevich, ``Online convex programming and generalized infinitesimal
  gradient ascent,'' in \emph{Proc. Intl. Conf. Machine Learning}, Washington
  D.C., 2003, pp. 928--936.

\end{thebibliography}

\clearpage
\onecolumn
\appendix
\begin{center}
{\large \bf Supplementary Document for ``Bandit Online Learning with Unknown Delays''}
\end{center}

\section{Real to virtual slot mapping}\label{sec.A}

For the analysis, let $t(\tau)$ denote the real slot when the real loss $\bm{l}_{t(\tau)}$ corresponding to $\tilde{\bm{l}}_\tau$ was incurred, i.e., $\tilde{\bm{l}}_\tau = \hat{\bm{l}}_{t(\tau)|t(\tau) + d_{t(\tau)}} $. Also define an auxiliary variable $\tilde{s}_\tau = \tau - 1 - L_{t(\tau)-1}$. See an example in Fig. \ref{fig.RVlearner_appdix} and Table \ref{tab.para}.

\begin{lemma}\label{lemma.index}
The following relations hold: i) $\tilde{s}_\tau \geq 0, ~ \forall \tau$; ii) $\sum_{\tau = 1}^T \tilde{s}_\tau = \sum_{t=1}^T d_t$; and, iii) if $\max_t d_t \leq \bar{d}$, we have
$\tilde{s}_\tau \leq 2 \bar{d}, ~ \forall \tau$.
\end{lemma}

\begin{proof}
We first prove the property i) $\tilde{s}_\tau \geq 0, ~ \forall t$. Consider at virtual slot $\tau$, the observed loss is $l_{t(\tau)}(a_{t(\tau)})$ with the corresponding $\tilde{s}_\tau = \tau - 1 - L_{t(\tau)-1}$. Suppose that $ L_{t(\tau)-1} = m$, where $0 \leq m \leq t(\tau)-1$ (by definition of $L_{t(\tau)-1}$). The history $ L_{t(\tau)-1} = m$ suggests that at the beginning of $t_1 = t(\tau)$, there are in total $m$ received feedback. On the other hand, the loss $l_{t(\tau)}(a_{t(\tau)})$ is observed at the end of slot $t_2 = t(\tau) + d_{t(\tau)} \geq t_1$, thus at the beginning of $t_2$, there are at least $m$ observations. Hence we must have $\tau \geq m+1$. Then by the definition, $\tilde{s}_\tau \geq m+1 -1 - m= 0 $.
	
	Then for the property ii) $\sum_{\tau = 1}^T \tilde{s}_\tau = \sum_{t=1}^T d_t$, the proof follows from the definition of $\tilde{s}_\tau$, i.e.,
	\begin{align}
		\sum_{\tau = 1}^T \tilde{s}_\tau &= \sum_{\tau = 1}^T \Big( \tau -1 - L_{t(\tau)-1} \Big) = \sum_{t = 1}^T ( t -1) - \sum_{\tau = 1}^T L_{t(\tau)-1} \nonumber \\
		& \stackrel{(a)}{=} \sum_{t = 1}^T \big( t -1 - L_{t-1} \big) \stackrel{(b)}{=} \sum_{t=1}^T d_t 
	\end{align} 
	where (a) is due to the fact that $\{t(\tau)\}_{\tau = 1}^T$ is a permutation of $\{1,\cdots,T\}$; and (b) follows from the definition of $L_{t-1}$. 
	
	Finally, for the property iii), 
	notice that $L_{t(\tau)-1} \geq t(\tau) - 1 -\bar{d}$, which follows that at the beginning of $t = t(\tau) $, the losses of slots $t\leq t(\tau) - 1 -\bar{d}$ must have been received. Therefore, we have
	\begin{equation}
		\tilde{s}_\tau = \tau - 1 - L_{t(\tau)-1} \leq \tau -1 -t(\tau) +1 +\bar{d} \stackrel{(c)}{\leq} 2\bar{d}
	\end{equation}
	where (c) follows that $l_{t(\tau)}(a_{t(\tau)})$ is observed at the end of $t = t(\tau) + d_{t(\tau)}$, and $L_{ t(\tau) + d_{t(\tau)}-1}$ is at most $t(\tau) + d_{t(\tau)} - 2$ (since $l_{t(\tau)}(a_{t(\tau)})$ is not observed), leading to the fact that $\tau$ is at most $t(\tau) + d_{t(\tau)}$, leading to $\tau -t(\tau) \leq d_{t(\tau)} \leq \bar{d} $.
\end{proof}

\vspace{-0.5cm}
\makeatletter\def\@captype{figure}\makeatother
\begin{minipage}{.55\textwidth}
\centering
\includegraphics[height=0.50\textwidth]{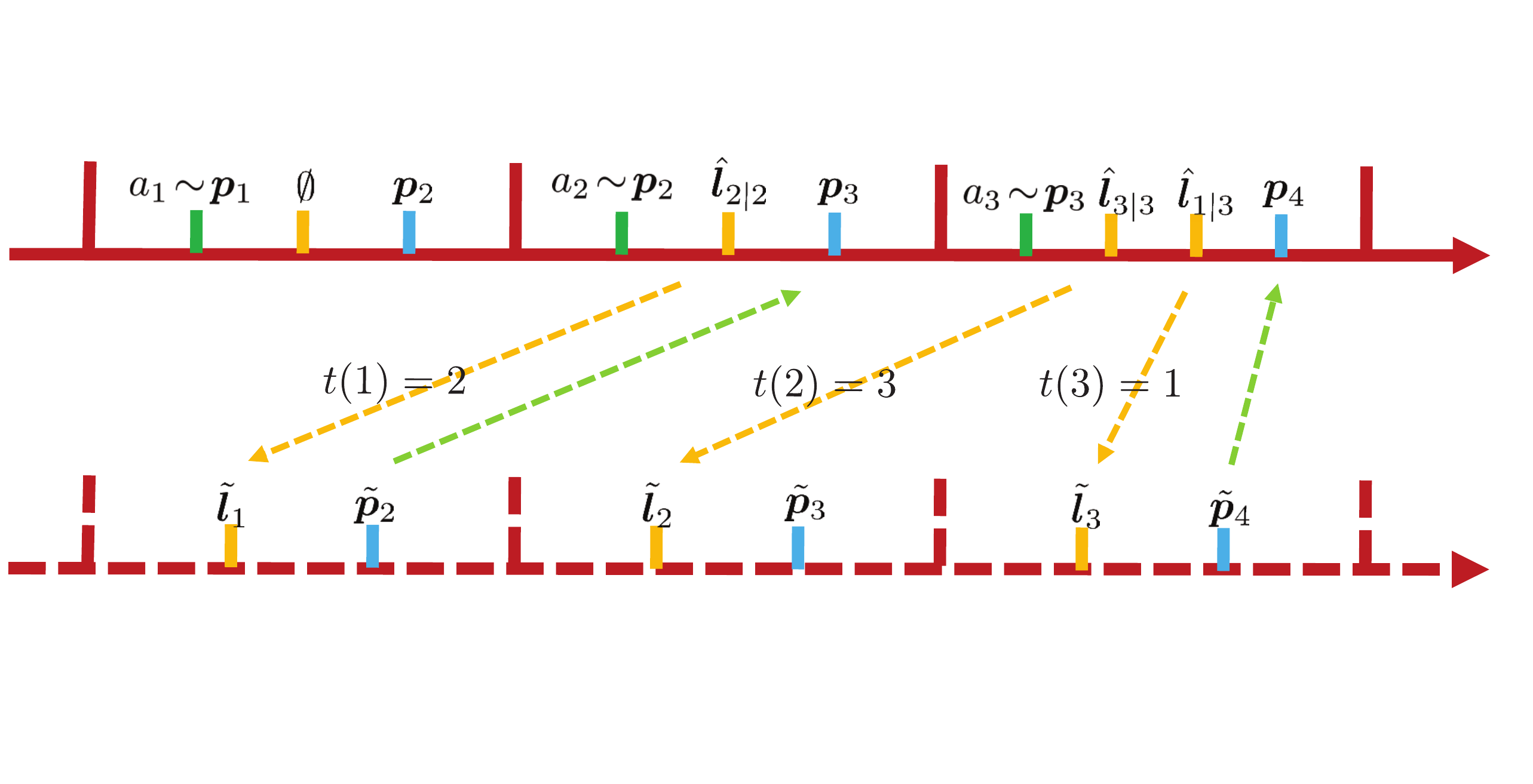}
\vspace{-0.9cm}
\caption{ An examples of mapping from real slots (solid line) and virtual slots (dotted line). The value of $t(\tau)$ is marked besides the corresponding yellow arrow; $T= 3$ with delay $d_1 = 2$, $d_2 = 0$, and $d_3 = 0$. }
\label{fig.RVlearner_appdix}
\end{minipage}
\quad
\makeatletter\def\@captype{table}\makeatother
\begin{minipage}{.4\textwidth}
\centering
\begin{tabular}{ c *{3}{|c}}
		\hline
		Virtual slot   & $\tau = 1$ & $\tau =2$  & $\tau =3$  \\ 
		\hline
		$t(\tau)$ & 2 & 	3  & 1   \\
		\hline
		$L_{t(\tau)-1}$ & 0 & 1  & 0   \\
		\hline
		$\tilde{s}_\tau$ & 0 & 0  & 2  \\
		\hline
\end{tabular} 
\caption{The value of $t(\tau)$, $L_{ t(\tau)-1}$, and $\tilde{s}_\tau$ in Fig. \ref{fig.RVlearner_appdix}. }
\label{tab.para}
\end{minipage}


\section{Proofs for DEXP3}\label{appendix.DEXP3}


Before diving into the proofs, we first show some useful yet simple bounds for different parameters of the DEXP3's (in virtual slots). In virtual slot $\tau$, the update is carried out the same as \eqref{eq.tildew1}, \eqref{eq.w1} and \eqref{eq.tildep1}, given by
\begin{equation}\label{eq.tildew}
	\tilde{w}_{\tau+1} (k) = \tilde{p}_\tau(k) \exp \Big[-\eta \min\big{\{} \delta_1,\tilde{l}_\tau (k)\big{\}}\Big],~ \forall k,
\end{equation}
\begin{equation}\label{eq.w}
	w_{\tau+1} (k) = \max \bigg{\{} \frac{\tilde{w}_{\tau+1} (k)}{\sum_{j=1}^K \tilde{w}_{\tau+1} (j)}, \frac{\delta_2}{K}\bigg{\}},~ \forall k,
\end{equation}
\begin{equation}\label{eq.tildep}
	\tilde{p}_{\tau+1} (k) = \frac{w_{\tau+1} (k)}{\sum_{j=1}^K w_{\tau+1} (j)},~\forall k.
\end{equation}

Since $\tilde{l}_\tau (k) \geq 0, \forall  k, \tau$, we have
\begin{equation}\label{eq.ineq1}
	\sum_{j=1}^K \tilde{w}_\tau (j) \leq \sum_{j=1}^K \tilde{p}_{\tau-1} (j) =1.
\end{equation}
And $\sum_{k=1}^K w_\tau(k)$ is bounded by
\begin{equation}\label{eq.ineq2}
	\sum_{k=1}^K w_\tau (k) \geq \sum_{k=1}^K \frac{\tilde{w}_\tau(k)}{\sum_{j=1}^K \tilde{w}_\tau(j)} = 1;
\end{equation}
\begin{equation}\label{eq.ineq3}
	\sum_{k=1}^K w_\tau (k) \leq \sum_{k=1}^K \frac{\tilde{w}_\tau(k)}{\sum_{j=1}^K \tilde{w}_\tau(j)} + \delta_2 = 1 + \delta_2.
\end{equation}
Finally, $\tilde{p}_\tau (k)$ is bounded by
\begin{equation}\label{eq.ineq4}
  \frac{\delta_2}{K(1+\delta_2)}	\leq \frac{w_\tau (k)}{1+\delta_2}\leq \tilde{p}_\tau(k) \leq w_\tau(k).
\end{equation}

\subsection{Proof of Lemma \ref{lemma.small/big}}

\begin{lemma}\label{lemma.small-big}
In consecutive virtual slots $\tau-1$ and $\tau$, the following inequality holds for any $k$.
	\begin{equation}
		\tilde{p}_{\tau - 1}(k) - \tilde{p}_\tau (k) \leq \tilde{p}_{\tau -1}(k) \frac{\delta_2 +\eta \min\big{\{}\delta_1,\tilde{l}_{\tau - 1}(k)\big{\}} }{1+\delta_2}.
	\end{equation}
\end{lemma}
\begin{proof}
First, we have
	\begin{align}
	\tilde{p}_\tau (k) \stackrel{(a)}{\geq} \frac{w_\tau (k)}{1+\delta_2} \geq \frac{\tilde{w}_\tau(k)}{\sum_{j=1}^K \tilde{w}_\tau(j)(1+\delta_2)}  \stackrel{(b)}{\geq}	\frac{\tilde{w}_\tau(k)}{1+\delta_2} = 	\frac{\tilde{p}_{\tau-1}(k) \exp \Big[-\eta \min\big{\{}\delta_1,\tilde{l}_{\tau - 1}(k)\big{\}}\Big]}{1+\delta_2}
	\end{align}
	where (a) is the result of \eqref{eq.ineq4}; (b) is due to \eqref{eq.ineq1}. Hence, we have 
	\begin{align}\label{eq.small-big.1}
	\tilde{p}_\tau (k) - \tilde{p}_{\tau -1} (k) & \geq \frac{\tilde{p}_{\tau-1}(k) \exp \Big[-\eta \min\big{\{}\delta_1,\tilde{l}_{\tau - 1}(k)\big{\}}\Big]}{1+\delta_2} - \tilde{p}_{\tau -1} (k) \nonumber \\
	& \stackrel{(c)}{\geq} \frac{\tilde{p}_{\tau -1} (k)}{1+\delta_2} \Big[ 1-\eta \min\big{\{}\delta_1,\tilde{l}_{\tau - 1}(k)\big{\}}\Big] - \tilde{p}_{\tau -1} (k) \nonumber \\
	& = \tilde{p}_{\tau -1}(k) \frac{-\delta_2 -\eta \min\big{\{}\delta_1,\tilde{l}_{\tau - 1}(k)\big{\}} }{1+\delta_2}
	\end{align}
	where (c) follows from $e^{-x} \geq 1-x$ and the proof is completed by multiplying $-1$ on both sides of \eqref{eq.small-big.1}.
\end{proof}

%
	From Lemma \ref{lemma.small-big}, we have 
	\begin{align}
		\tilde{p}_{\tau - 1}(k) - \tilde{p}_\tau (k) \leq \tilde{p}_{\tau -1}(k) \frac{\delta_2 +\eta \min\big{\{}\delta_1,\tilde{l}_{\tau - 1}(k)\big{\}} }{1+\delta_2} \leq \tilde{p}_{\tau -1}(k)  \big(\delta_2 +\eta \delta_1 ).
	\end{align}
	Hence, as long as $1-\delta_2 - \eta \delta_1 \geq 0$, we can guarantee that \eqref{eq.small/big.1} is satisfied. 

\subsection{Proof of Lemma \ref{lemma.big/small}}
\begin{lemma}\label{lemma.big-small}
	The following inequality holds for any $\tau$ and any $k$
	\begin{equation}
		\tilde{p}_{\tau}(k) - \tilde{p}_{\tau-1}(k) \leq \tilde{p}_{\tau}(k) \bigg[ 1 - I_\tau (k) \sum_{j=1}^K \tilde{p}_{\tau-1}(j) \Big(1 - \eta \min\big{\{}\delta_1,\tilde{l}_{\tau - 1}(j)\big{\}} \Big) \bigg]
	\end{equation}
	where $I_\tau (k) := \mathds{1} \big(w_\tau(k) > \frac{\delta_2}{K} \big)$.
\end{lemma}
\begin{proof}
	We first show that 
	\begin{equation}\label{eq.big-small.1}
		\tilde{w}_\tau (k) \geq \tilde{p}_\tau(k) I_\tau(k) \sum_{j=1}^{K} \tilde{w}_\tau (j).
	\end{equation}
	It is easy to see that inequality \eqref{eq.big-small.1} holds when $I_\tau(k) = 0$. When $I_\tau(k) = 1$, we have $w_\tau(k) = \tilde{w}_\tau (k) / \big( \sum_{j=1}^{K} \tilde{w}_\tau (j)\big)$. By \eqref{eq.ineq4}, we have $\tilde{p}_\tau (k) \leq w_\tau(k) = \tilde{w}_\tau (k) / \big( \sum_{j=1}^{K} \tilde{w}_\tau (j)\big)$, from which \eqref{eq.big-small.1} holds. Then we have 
	\begin{align}
		\tilde{p}_{\tau}(k) &- \tilde{p}_{\tau -1} (k)  \leq \tilde{p}_{\tau}(k) - \tilde{w}_\tau (k) \leq \tilde{p}_{\tau}(k) - \tilde{p}_\tau(k) I_\tau(k) \sum_{j=1}^{K} \tilde{w}_\tau (j) \nonumber \\
		& = \tilde{p}_{\tau}(k) \bigg[ 1- I_\tau(k) \sum_{j=1}^{K} \tilde{w}_\tau (j)  \bigg] =  \tilde{p}_{\tau}(k) \bigg{\{} 1- I_\tau(k) \sum_{j=1}^{K} \tilde{p}_{\tau-1} (j) \exp \Big[-\eta \min\big{\{}\delta_1,\tilde{l}_{\tau - 1}(j)\big{\}}\Big]  \bigg{\}} \nonumber \\
		& \stackrel{(a)}{\leq} \tilde{p}_{\tau}(k) \bigg[ 1 - I_\tau (k) \sum_{j=1}^K \tilde{p}_{\tau-1}(j) \Big(1 - \eta \min\big{\{}\delta_1,\tilde{l}_{\tau - 1}(j)\big{\}} \Big) \bigg]
	\end{align}
	where in (a) we used $e^{-x} \geq 1-x$.
\end{proof}

	The proof of Lemma \ref{lemma.big/small} builds on Lemma \ref{lemma.big-small}. First consider the case of $I_\tau(k)=0$. In this case Lemma \ref{lemma.big-small} becomes $\tilde{p}_{\tau}(k) - \tilde{p}_{\tau-1}(k) \leq \tilde{p}_{\tau}(k)$, which is trivial. On the other hand, since $I_\tau(k)=0$, we have $w_\tau (k) = \frac{\delta_2}{K}$. Then leveraging \eqref{eq.ineq4}, we have $\tilde{p}_\tau (k) \leq w_\tau (k) = \frac{\delta_2}{K}$.
Plugging the lower bound of $\tilde{p}_{\tau -1}(k)$ into \eqref{eq.ineq4}, we have 
	\begin{equation}\label{eq.coro.32}
		\frac{\tilde{p}_\tau (k)}{\tilde{p}_{\tau-1} (k)} \leq \frac{\delta_2}{K} \frac{1}{\tilde{p}_{\tau-1} (k)} \leq \frac{\delta_2}{K}  \frac{K(1+\delta_2)}{\delta_2} = 1+ \delta_2.
	\end{equation}
	
Considering the case of $I_\tau(k)=1$, Lemma \ref{lemma.big-small} becomes
	\begin{align}\label{eq.coro.33}
		\tilde{p}_{\tau}(k) - \tilde{p}_{\tau-1}(k)& \leq \tilde{p}_{\tau}(k) \bigg[ 1 -  \sum_{j=1}^K \tilde{p}_{\tau-1}(j) \Big(1 - \eta \min\big{\{}\delta_1,\tilde{l}_{\tau - 1}(j)\big{\}} \Big) \bigg] \nonumber \\
		& = \eta \tilde{p}_{\tau}(k)\sum_{j=1}^K \tilde{p}_{\tau-1}(j) \min\big{\{}\delta_1,\tilde{l}_{\tau - 1}(k)\big{\}} \leq \eta \tilde{p}_{\tau}(k) \delta_1.
	\end{align}
	Rearranging \eqref{eq.coro.33} and combining it with \eqref{eq.coro.32}, we complete the proof.


\subsection{Proof of Lemma \ref{lemma.innerreg}}
For conciseness, define $\tilde{\bm{c}}_\tau:= \min \big{\{} \tilde{\bm{l}}_\tau, \delta_1 \cdot \bm{1} \}$	, and correspondingly $\tilde{c}_\tau (k):= \min \{ \tilde{l}_\tau(k), \delta_1 \big{\}}$. We further define $\tilde{W}_\tau := \sum_{k=1}^K \tilde{w}_\tau (k)$, and $W_\tau := \sum_{k=1}^K w_\tau (k)$. Leveraging these auxiliary variables, we have 
\begin{align}\label{eq.innerreg.1}
	\tilde{W}_{T+1}& = \sum_{k=1}^K \tilde{w}_{T+1}(k) = \sum_{k=1}^K \tilde{p}_{T}(k) \exp \big[-\eta \tilde{c}_T (k)\big] = \sum_{k=1}^K \frac{w_T(k)}{W_T} \exp \big[-\eta \tilde{c}_T (k)\big] \nonumber \\
	& \geq \sum_{k=1}^K \frac{\tilde{w}_T(k)}{\tilde{W}_T}  \frac{\exp \big[-\eta \tilde{c}_T (k)\big]}{W_T} = \sum_{k=1}^K \tilde{p}_{T-1}(k) \frac{\exp \big[-\eta \tilde{c}_T (k) -\eta \tilde{c}_{T-1} (k) \big]}{\tilde{W}_T W_T} \nonumber \\
	& = \sum_{k=1}^K \frac{w_{T-1}(k)}{W_{T-1}} \frac{\exp \big[-\eta \tilde{c}_T (k) -\eta \tilde{c}_{T-1} (k) \big]}{\tilde{W}_T W_T} \geq \cdots  \geq \sum_{k=1}^K \frac{\tilde{w}_1 (k) \exp \Big[ -\eta \sum_{\tau =1}^{T} \tilde{c}_\tau (k) \Big]}{\prod_{\tau =1}^{T} \big(W_\tau \tilde{W}_\tau \big)}.
\end{align}

Then, for any probability distribution $\bm{p} \in \Delta_K$
noticing that the initialization of $\tilde{w}_1(k) = 1, \forall k$ and hence $\tilde{W}_1 = K$, inequality \eqref{eq.innerreg.1} implies that
\begin{align}\label{eq.innerreg.2}
	&\sum_{k=1}^K p(k) \exp \Big[ -\eta \sum_{\tau =1}^{T} \tilde{c}_\tau (k) \Big] \leq \sum_{k=1}^K  \exp \Big[ -\eta \sum_{\tau =1}^{T} \tilde{c}_\tau (k) \Big]  \leq  \tilde{W}_1 \prod_{\tau =1}^{T} \big(W_\tau \tilde{W}_{\tau+1}\big) \stackrel{(a)}{\leq} K (1+\delta_2)^T \prod_{\tau =2}^{T+1} \tilde{W}_\tau,
\end{align}
where in (a) we used the fact that $W_\tau \leq 1+ \delta_2 $. Then, using the the Jensen's inequality on $e^{-x}$, we have
\begin{equation}\label{eq.innerreg.3}
	\sum_{k=1}^K p(k) \exp \Big[ -\eta \sum_{\tau =1}^{T} \tilde{c}_\tau (k) \Big] \geq \exp \bigg[ -\eta \sum_{k=1}^K \sum_{\tau =1}^{T} p(k)\tilde{c}_\tau (k) \bigg].
\end{equation}
Plugging \eqref{eq.innerreg.3} into \eqref{eq.innerreg.2}, we arrive at
\begin{equation}\label{eq.innerreg.4}
	\exp \bigg[ -\eta \sum_{k=1}^K \sum_{\tau =1}^{T} p(k)\tilde{c}_\tau (k) \bigg] \leq  K (1+\delta_2)^T \prod_{\tau =2}^{T+1} \tilde{W}_\tau.
\end{equation}

On the other hand, $\tilde{W}_\tau$ can be upper bounded by
\begin{align}\label{eq.innerreg.5}
	\tilde{W}_\tau & = \sum_{k=1}^K \tilde{w}_\tau = \sum_{k=1}^K \tilde{p}_{\tau-1} (k) \exp \big[-\eta \tilde{c}_{\tau -1}(k)  \big] \nonumber \\
	& \stackrel{(b)}{\leq} \sum_{k=1}^K \tilde{p}_{\tau-1}(k) \bigg(1 - \eta  \tilde{c}_{\tau -1}(k) + \frac{\eta^2}{2} \big[ \tilde{c}_{\tau -1}(k)\big]^2 \bigg) \nonumber \\
	& = 1 - \eta \sum_{k=1}^K \tilde{p}_{\tau-1}(k) \tilde{c}_{\tau -1}(k) + \frac{\eta^2}{2} \sum_{k=1}^K \tilde{p}_{\tau-1}(k) \big[ \tilde{c}_{\tau -1}(k)\big]^2
\end{align}
where (b) follows from $e^{-x} \leq 1-x+x^2/2, ~\forall x \geq 0 $. Taking logarithm on both sides of \eqref{eq.innerreg.5}, we arrive at
\begin{align}\label{eq.innerreg.6}
	\ln \tilde{W}_\tau & \leq \ln \bigg( 1 - \eta \sum_{k=1}^K \tilde{p}_{\tau-1}(k) \tilde{c}_{\tau -1}(k) + \frac{\eta^2}{2} \sum_{k=1}^K \tilde{p}_{\tau-1}(k) \big[ \tilde{c}_{\tau -1}(k)\big]^2 \bigg) \nonumber \\
	& \stackrel{(c)}{\leq} - \eta \sum_{k=1}^K \tilde{p}_{\tau-1}(k) \tilde{c}_{\tau -1}(k) + \frac{\eta^2}{2} \sum_{k=1}^K \tilde{p}_{\tau-1}(k) \big[ \tilde{c}_{\tau -1}(k)\big]^2
\end{align}
where (c) follows from $\ln(1+x) \leq x$. Then taking logarithm on both sides of \eqref{eq.innerreg.4} and plugging \eqref{eq.innerreg.6} in, we arrive at
\begin{align}\label{eq.innerreg.7}
	-\eta \sum_{k=1}^K \sum_{\tau =1}^{T} p(k)\tilde{c}_\tau (k) \leq T \ln (1+\delta_2) + \ln K  - \eta \sum_{\tau =1 }^T\sum_{k=1}^K \tilde{p}_{\tau}(k) \tilde{c}_{\tau }(k) + \frac{\eta^2}{2} \sum_{\tau =1 }^T\sum_{k=1}^K \tilde{p}_{\tau}(k) \big[ \tilde{c}_{\tau }(k)\big]^2.
\end{align}
Rearranging the terms of \eqref{eq.innerreg.7} and writing it compactly, we obtain
\begin{align}
	\sum_{\tau =1 }^T  \big(\tilde{\bm{p}}_\tau - \bm{p} \big)^\top \tilde{\bm{c}}_\tau  \leq & \frac{T \ln (1\!+\!\delta_2)\!+\! \ln K}{\eta} + \frac{\eta}{2} \sum_{\tau=1}^T \sum_{k=1}^K \tilde{p}_\tau (k) \big[ \tilde{c}_\tau (k) \big]^2\nonumber\\
	 \leq &\frac{T \ln (1\!+\!\delta_2)\!+ \!\ln K}{\eta} + \frac{\eta}{2} \sum_{\tau=1}^T \sum_{k=1}^K \tilde{p}_\tau (k) \big[ \tilde{l}_\tau (k) \big]^2.
\end{align}


\subsection{Proof of Theorem \ref{theo.reg}}


 	To begin with, the instantaneous regret can be written as
	\begin{align}\label{eq.perslotreg}
	 	\bm{p}_t^\top \bm{l}_t - \bm{p}^\top \bm{l}_t &= \sum_{k=1}^K p_t(k) l_t(k) - \sum_{k=1}^K p(k)l_t(k) \nonumber \\
	 	& \stackrel{(a)}{=} \sum_{k=1}^K p_t(k) \mathbb{E}_{a_t}\bigg[\frac{l_t(k) \mathds{1} (a_t = k)}{p_t(k)} \bigg] - \sum_{k=1}^K p(k) \mathbb{E}_{a_t}\bigg[\frac{l_t(k) \mathds{1} (a_t = k)}{p_t(k)} \bigg] \nonumber \\
		& = \sum_{k=1}^K \big( p_t(k) - p(k)\big) \mathbb{E}_{a_t} \bigg[ \frac{l_t(k) \mathds{1} (a_t = k)}{p_{t+d_t}(k)} \frac{p_{t+d_t} (k)}{p_t(k)} \bigg] \nonumber \\
		& \leq \max_{k} \frac{p_{t+d_t} (k)}{p_t(k)} \sum_{k=1}^K \big( p_t(k) - p(k)\big) \mathbb{E}_{a_t} \bigg[ \frac{l_t(k) \mathds{1} (a_t = k)}{p_{t+d_t}(k)} \bigg] \nonumber \\
		& \stackrel{(b)}{=} \bigg( \max_{k} \frac{p_{t+d_t} (k)}{p_t(k)} \bigg) \mathbb{E}_{a_t} \Big[ \bm{p}_t^\top \hat{\bm{l}}_{t|t+{d_t}} - \bm{p}^\top \hat{\bm{l}}_{t|t+{d_t}} \Big]
	\end{align}
where (a) is due to $\mathbb{E}_{a_t}\bigg[\frac{l_t(k) \mathds{1} (a_t = k)}{p_t(k)} \bigg] = l_t(k)$, 
and (b) follows from $\hat{l}_{t|t+{d_t}}(k) = \frac{l_t(k) \mathds{1} (a_t = k)}{p_{t+d_t}(k)}$.

Then the overall regret of $T$ slots is given by
\begin{align}\label{eq.reg.1}
	\text{Reg}_T &= \mathbb{E} \bigg[ \sum_{t=1}^T \bm{p}_t^\top \bm{l}_t \bigg] - \bm{p}^\top \bm{l}_t \leq\mathbb{E} \bigg[ \sum_{t=1}^T \bigg( \max_{k} \frac{p_{t+d_t} (k)}{p_t(k)} \bigg) \mathbb{E}_{a_t} \big[ \bm{p}_t^\top \hat{\bm{l}}_{t|t+{d_t}} - \bm{p}^\top \hat{\bm{l}}_{t|t+{d_t}} \big] \bigg] \nonumber \\
	& \stackrel{(c)}{=} \mathbb{E} \bigg[ \sum_{\tau=1}^T \bigg( \max_{k} \frac{p_{ t(\tau)+d_{t(\tau)}} (k)}{p_{t(\tau)}(k)} \bigg) \mathbb{E}_{a_{t(\tau)}} \Big[ \bm{p}_{t(\tau)}^\top \hat{\bm{l}}_{t(\tau)|t(\tau) + d_t(\tau) } - \bm{p}^\top \hat{\bm{l}}_{t(\tau)|t(\tau) + d_t(\tau) } \Big] \bigg] \nonumber \\
	& \stackrel{(d)}{=} \mathbb{E} \bigg[\sum_{\tau=1}^T \bigg( \max_{k} \frac{p_{ t(\tau)+d_{t(\tau)} }(k)}{p_{t(\tau)}(k)} \bigg) \mathbb{E}_{a_{t(\tau)}} \Big[ \bm{p}_{t(\tau)}^\top \tilde{\bm{l}}_\tau - \bm{p}^\top \tilde{\bm{l}}_\tau \Big] \bigg] \nonumber \\
	& \stackrel{(e)}{=} \mathbb{E} \bigg[ \sum_{\tau=1}^T \bigg( \max_{k} \frac{p_{ t(\tau)+d_{t(\tau)} }(k)}{p_{t(\tau)}(k)} \bigg) \mathbb{E}_{a_{t(\tau)}} \Big[ \tilde{\bm{p}}_{\tau - \tilde{s}_\tau}^\top \tilde{\bm{l}}_\tau - \bm{p}^\top \tilde{\bm{l}}_\tau \Big] \bigg] \nonumber \\
	& = \mathbb{E} \bigg[ \sum_{\tau=1}^T \bigg( \max_{k} \frac{p_{ t(\tau)+d_{t(\tau)} }(k)}{p_{t(\tau)}(k)} \bigg) \bigg( \mathbb{E}_{a_{t(\tau)}} \Big[ \tilde{\bm{p}}_{\tau - \tilde{s}_\tau}^\top \tilde{\bm{l}}_\tau - \tilde{\bm{p}}_{\tau}^\top \tilde{\bm{l}}_\tau \Big]  + \mathbb{E}_{a_{t(\tau)}} \Big[ \tilde{\bm{p}}_{\tau}^\top \tilde{\bm{l}}_\tau - \bm{p}^\top \tilde{\bm{l}}_\tau \Big]  \bigg) \bigg]
\end{align}
where (c) is due to the fact that $\{ t(1), t(2), \ldots, t(T)\}$ is a permutation of $\{1,2,\ldots,T\}$; (d) follows from $\tilde{\bm{l}}_\tau = \hat{\bm{l}}_{t(\tau)|t(\tau) + d_t(\tau) } $; (e) uses the fact $\bm{p}_t = \tilde{\bm{p}}_{L_{t-1}+1}$ and $\bm{p}_{t(\tau)} = \tilde{\bm{p}}_{L_{t(\tau)-1}+1} = \tilde{\bm{p}}_{\tau - \tilde{s}_\tau}$. 

First note that between real time slot $t(\tau)$ and $t(\tau)+d_{t(\tau)}$, there is at most $\bar{d}+d_{t(\tau)}\leq 2\bar{d}$ feedback received. Hence the corresponding virtual slots will not differ larger than $2\bar{d}$. Note also that the index of virtual slot corresponding to $t(\tau)$ must be no larger than that of $t(\tau)+d_{t(\tau)}$. Hence we have for all $\tau \in [1,T]$,
\begin{align}\label{eq.reg.A}
	\max_{k} \frac{p_{ t(\tau)+d_{t(\tau)} }(k)}{p_{t(\tau)}(k)} \leq \bigg( \max_k \frac{ \tilde{p}_{\tau+1}(k) }{\tilde{p}_\tau (k)} \bigg)^{2\bar{d}} \stackrel{(f)}{\leq} \max \bigg{\{}(1+\delta_2)^{2\bar{d}} , \frac{1}{(1-\eta \delta_1)^{2\bar{d}} } \bigg{\}}
\end{align}
where (f) is the result of Lemma \ref{lemma.big/small}.

Then, to bound the terms in the second brackets of \eqref{eq.reg.1}, again we denote $\tilde{\bm{c}}_\tau:= \min \big{\{} \tilde{\bm{l}}_\tau, \delta_1\cdot \bm{1} \}$, and correspondingly $\tilde{c}_\tau (k):= \min \{ \tilde{l}_\tau(k), \delta_1 \big{\}}$ for conciseness. Then we have
\begin{align}\label{eq.reg.2}
	\tilde{\bm{p}}_{\tau - \tilde{s}_\tau}^\top \tilde{\bm{c}}_\tau - &\tilde{\bm{p}}_{\tau}^\top \tilde{\bm{c}}_\tau = \tilde{\bm{c}}_\tau^\top (\tilde{\bm{p}}_{\tau - \tilde{s}_\tau} - \tilde{\bm{p}}_{\tau})  \stackrel{(g)}{=} \tilde{c}_\tau (m) \sum_{j=0}^{\tilde{s}_\tau -1} \big( \tilde{p}_{\tau - \tilde{s}_\tau + j}(m) - \tilde{p}_{\tau - \tilde{s}_\tau + j +1} (m)\big) \nonumber \\
	& \stackrel{(h)}{\leq} \tilde{c}_\tau (m) \sum_{j=0}^{\tilde{s}_\tau -1} \tilde{p}_{\tau - \tilde{s}_\tau + j}(m) \frac{ \delta_2 +\eta  \tilde{c}_{\tau - \tilde{s}_\tau + j}(m)}{1+\delta_2} \leq \tilde{c}_\tau (m)  \sum_{j=0}^{\tilde{s}_\tau -1} \big( \eta \tilde{p}_{\tau - \tilde{s}_\tau + j}(m) \tilde{c}_{\tau - \tilde{s}_\tau + j}(m) + \delta_2  \big) \nonumber \\
	& \leq \tilde{l}_\tau (m)  \sum_{j=0}^{\tilde{s}_\tau -1} \big( \eta \tilde{p}_{\tau - \tilde{s}_\tau + j}(m) \tilde{l}_{\tau - \tilde{s}_\tau + j}(m) + \delta_2  \big)
	\end{align}
where (g) follows from the facts that $\tilde{\bm{l}}_\tau$ has at most one entry (with index $m$) being non-zero [cf. \eqref{eq.add1}] and $\tilde{s}_\tau \geq 0$ [cf. Lemma \ref{lemma.index}]; and (h) is the result of Lemma \ref{lemma.small-big}. Then notice that 
\begin{align}\label{eq.reg.3}
	\tilde{l}_\tau (k) \tilde{p}_\tau (k) = \frac{ l_{t(\tau)}(k) }{ p_{t(\tau) + d_t(\tau)} (k)}\tilde{p}_\tau (k) \stackrel{(i)}{\leq} \bigg( \max_k \frac{\tilde{p}_{\tau}(k)}{\tilde{p}_{\tau+1}(k)} \bigg)^{2\bar{d}} \leq \frac{1}{(1-\delta_2 - \eta \delta_1)^{2\bar{d}}}
\end{align}
where (i) uses the fact that between $t(\tau)$ and $t(\tau) + d_{t(\tau)}$ there is at most $2\bar{d}$ feedback; then further applying the result of Lemma \ref{lemma.small/big}, inequality \eqref{eq.reg.3} can be obtained. Plugging \eqref{eq.reg.3} back in to \eqref{eq.reg.2} and taking expectation w.r.t. $a_{t(\tau)}$, we arrive at 
\begin{align}
	\mathbb{E}_{a_{t(\tau)}} \big[ \tilde{\bm{p}}_{\tau - \tilde{s}_\tau}^\top \tilde{\bm{c}}_\tau - \tilde{\bm{p}}_{\tau}^\top \tilde{\bm{c}}_\tau \big]& \leq \bigg( \frac{\eta\tilde{s}_\tau}{(1-\delta_2 - \eta \delta_1)^{2\bar{d}}} +\delta_2\tilde{s}_\tau \bigg) \sum_{k=1}^K p_{t(\tau)}(k) \tilde{l}_\tau (k) \nonumber \\
	&\stackrel{(j)}{\leq} K \frac{1}{(1-\delta_2 - \eta \delta_1)^{2\bar{d}}}\bigg( \frac{\eta\tilde{s}_\tau}{(1-\delta_2 - \eta \delta_1)^{2\bar{d}}} +\delta_2 \tilde{s}_\tau \bigg)
\end{align}
where (j) follows a similar reason of \eqref{eq.reg.3}. Then, noticing $\sum_{\tau = 1}^T \tilde{s}_\tau = \sum_{t=1}^T d_t = D$, we have 
\begin{equation}\label{eq.reg.B}
	\sum_{\tau = 1}^T \mathbb{E}_{a_{t(\tau)}} \big[ \tilde{\bm{p}}_{\tau - \tilde{s}_\tau}^\top \tilde{\bm{c}}_\tau - \tilde{\bm{p}}_{\tau}^\top \tilde{\bm{c}}_\tau \big] \leq  \frac{KD}{(1-\delta_2 - \eta \delta_1)^{2\bar{d}}}\bigg( \frac{\eta}{(1-\delta_2 - \eta \delta_1)^{2\bar{d}}} +\delta_2 \bigg).
\end{equation}

Using a similar argument of \eqref{eq.reg.3}, we can obtain
\begin{align}
	\mathbb{E}_{a_{t(\tau)}} \Big[ \tilde{p}_\tau (k) \big[ \tilde{l}_\tau (k) \big]^2 \Big] = \tilde{p}_\tau (k) \frac{l_{t(\tau)}^2 (k)}{p^2_{t(\tau) + d_{t(\tau)}}(k)} p_{t(\tau)}(k) \leq \frac{1}{(1-\delta_2 - \eta \delta_1)^{4\bar{d}}}
\end{align}

Then leveraging Lemma \ref{lemma.innerreg}, we arrive at
\begin{align}\label{eq.reg.4}
	\sum_{\tau=1}^T \mathbb{E}_{a_{t(\tau)}} \big[ (\tilde{\bm{p}}_\tau - \tilde{\bm{p}} )^\top \tilde{\bm{c}}_\tau \big] &\leq \frac{T \ln (1+\delta_2)+ \ln K}{\eta} + \frac{\eta}{2} \sum_{\tau=1}^T \sum_{k=1}^K \mathbb{E}_{a_{t(\tau)}} \Big[ \tilde{p}_\tau (k) \big[ \tilde{l}_\tau (k) \big]^2 \Big] \nonumber \\ 
	&\leq \frac{T \ln (1+\delta_2)+ \ln K}{\eta} + \frac{\eta KT}{2(1-\delta_2 - \eta \delta_1)^{4\bar{d}}}.
\end{align}
	
The last step is to show that introducing $\delta_1$ will not incur too much extra regret.
Note that both $\tilde{\bm{c}}_\tau$ and $\tilde{\bm{l}}_\tau$ have only one entry being non-zero, whose index is denoted by $m_\tau$. Notice that $\tilde{l}_\tau(m_\tau) > \tilde{c}_\tau(m_\tau)$ only when $\tilde{l}_\tau(m_\tau) = \frac{l_{t(\tau)}(m_\tau)}{p_{t(\tau)+d_{t(\tau)} }(m_\tau)}> \delta_1$, which is equivalent to $p_{ t(\tau) + d_{t(\tau)} } (m_\tau) < l_{t(\tau)}(m_\tau)/\delta_1 \leq 1/\delta_1$. Hence, we have
\begin{align}\label{eq.reg.5}
	 \sum_{\tau=1}^T\mathbb{E}_{a_{t(\tau)}} \big[ (\tilde{\bm{p}}_\tau - \tilde{\bm{p}} )^\top \tilde{\bm{l}}_\tau \big] &= \sum_{\tau=1}^T \mathbb{E}_{a_{t(\tau)}} \big[ (\tilde{\bm{p}}_\tau - \tilde{\bm{p}} )^\top \tilde{\bm{c}}_\tau \big] + \sum_{\tau=1}^T \mathbb{E}_{a_{t(\tau)}} \big[ (\tilde{\bm{p}}_\tau - \tilde{\bm{p}} )^\top \big( \tilde{\bm{l}}_\tau - \tilde{\bm{c}}_\tau \big) \big] \nonumber \\
	& \stackrel{(h)}{\leq} \sum_{\tau=1}^T \mathbb{E}_{a_{t(\tau)}} \big[ (\tilde{\bm{p}}_\tau - \tilde{\bm{p}} )^\top \tilde{\bm{c}}_\tau \big]+ \sum_{\tau=1}^T \mathbb{E}_{a_{t(\tau)}} \Big[ \tilde{p}_\tau(m_\tau) \big( \tilde{l}_\tau(m_\tau) - \tilde{c}_\tau(m_\tau) \big) \mathds{1}\big(p_{t(\tau)+d_t(\tau)}(m_\tau)  < 1/\delta_1 \big) \Big] \nonumber \\
	& \leq \sum_{\tau=1}^T \mathbb{E}_{a_{t(\tau)}} \big[ (\tilde{\bm{p}}_\tau - \tilde{\bm{p}} )^\top \tilde{\bm{c}}_\tau \big]+ \sum_{\tau=1}^T \mathbb{E}_{a_{t(\tau)}} \Big[ \tilde{p}_\tau(m_\tau) \tilde{l}_\tau(m_\tau) \mathds{1}\big(p_{t(\tau)+d_t(\tau)}(m_\tau)  < 1/\delta_1 \big) \Big] 
\end{align}
where in (h), $m_\tau$ denotes the index of the only one none-zero entry of $\tilde{\bm{l}}_\tau$, and $\tilde{\bm{p}}$ is dropped due to the appearance of the indicator function. To proceed, notice that
\begin{align}\label{eq.reg.6}
	 &\mathbb{E}_{a_{t(\tau)}} \Big[ \tilde{l}_\tau(m_\tau) \tilde{p}_\tau(m_\tau) \mathds{1}\big(p_{t(\tau)+d_t(\tau)}(m_\tau )  < 1/\delta_1 \big) \Big] = \sum_{k=1}^K \frac{p_{t(\tau)} (k) l_{t(\tau)} (k)}{ p_{t(\tau) +d_t(\tau)} (k)}  \tilde{p}_\tau(k)\mathds{1}\big(p_{t(\tau)+d_t(\tau)}(k)  < 1/\delta_1 \big) \nonumber \\
	& \stackrel{(i)}{\leq}  \frac{\sum_{k=1}^K  \tilde{p}_\tau(k) \mathds{1}\big(p_{t(\tau)+d_t(\tau)}(k)  < 1/\delta_1 \big)}{(1-\delta_2 - \eta \delta_1)^{2\bar{d}}} = \sum_{k=1}^K \frac{\tilde{p}_\tau(k) }{p_{t(\tau)+d_t(\tau)}(k)} \frac{p_{t(\tau)+d_t(\tau)}(k)\mathds{1}\big(p_{t(\tau)+d_t(\tau)}(k)  < 1/\delta_1 \big)}{(1-\delta_2 - \eta \delta_1)^{2\bar{d}}} \nonumber \\
	& \stackrel{(j)}{\leq} \frac{K}{\delta_1 (1-\delta_2 - \eta \delta_1)^{4\bar{d}}}
\end{align}
where in (i) we used the a similar argument of \eqref{eq.reg.3}; and in (j) we used the fact $x \mathds{1}(x<a) \leq a$. 

Plugging \eqref{eq.reg.6} back into \eqref{eq.reg.5}, we arrive at
\begin{equation} \label{eq.reg.7}
	\sum_{\tau=1}^T \mathbb{E}_{a_{t(\tau)}} \Big[ (\tilde{\bm{p}}_\tau - \tilde{\bm{p}} )^\top \tilde{\bm{l}}_\tau \Big]  \leq \sum_{\tau=1}^T \mathbb{E}_{a_{t(\tau)}} \Big[ (\tilde{\bm{p}}_\tau - \tilde{\bm{p}} )^\top \tilde{\bm{c}}_\tau \Big]+ \frac{KT}{\delta_1 (1-\delta_2 - \eta \delta_1)^{4\bar{d}}}
\end{equation}

Applying similar arguments as \eqref{eq.reg.5} and \eqref{eq.reg.6}, we can also show that 
\begin{align} \label{eq.reg.8}
	\sum_{\tau = 1}^T \mathbb{E}_{a_{t(\tau)}} \Big[ \big( \tilde{\bm{p}}_{\tau - \tilde{s}_\tau}  - \tilde{\bm{p}}_{\tau} \big)^\top \tilde{\bm{l}}_\tau \Big]  \leq \sum_{\tau = 1}^T \mathbb{E}_{a_{t(\tau)}} \Big[ \big( \tilde{\bm{p}}_{\tau - \tilde{s}_\tau}  - \tilde{\bm{p}}_{\tau} \big)^\top \tilde{\bm{c}}_\tau \Big] + \frac{KD}{\delta_1 (1-\delta_2 - \eta \delta_1)^{6\bar{d}}}.
\end{align}

For the parameter selection, we have $T \ln (1+\delta_2) = T \ln (1+\frac{1}{T+D}) \leq \ln e = 1$.
Leveraging the inequality that $e \leq (1-2x)^{-2x}\leq 4, \forall x \in \mathds{N}^+$, we have that 
\begin{equation}\label{eq.add1}
	\frac{1}{(1 - \eta \delta_1)^{2\bar{d}}} \leq \frac{1}{(1-\delta_2 - \eta \delta_1)^{2\bar{d}}} = {\cal O} (1).
\end{equation}

From \eqref{eq.add1} it is not hard to see the bound on \eqref{eq.reg.A}, which is
\begin{align}\label{eq.reg.A1}
	\max_{k} \frac{p_{ t(\tau)+d_{t(\tau)} }(k)}{p_{t(\tau)}(k)}  \leq \max \bigg{\{}(1+\delta_2)^{2\bar{d}} , \frac{1}{(1-\eta \delta_1)^{2\bar{d}} } \bigg{\}} = {\cal O} (1).
\end{align}
Then for \eqref{eq.reg.B}, we have
\begin{equation}\label{eq.reg.B1}
	\sum_{\tau = 1}^T \mathbb{E}_{a_{t(\tau)}} \big[ \tilde{\bm{p}}_{\tau - \tilde{s}_\tau}^\top \tilde{\bm{c}}_\tau - \tilde{\bm{p}}_{\tau}^\top \tilde{\bm{c}}_\tau \big] \leq  \frac{KD}{(1-\delta_2 - \eta \delta_1)^{2\bar{d}}}\bigg( \frac{\eta}{(1-\delta_2 - \eta \delta_1)^{2\bar{d}}} +\delta_2 \bigg) = {\cal O}(\eta KD+ \delta_2 KD).
\end{equation}
For \eqref{eq.reg.4}, we have
\begin{align}\label{eq.reg.41}
	\sum_{\tau=1}^T \mathbb{E}_{a_{t(\tau)}} \big[ (\tilde{\bm{p}}_\tau - \tilde{\bm{p}} )^\top \tilde{\bm{c}}_\tau \big] \leq \frac{T \ln (1+\delta_2)+ \ln K}{\eta} + \frac{\eta KT}{2(1-\delta_2 - \eta \delta_1)^{4\bar{d}}} = {\cal O} \bigg(\eta KT  + \frac{1+ \ln K}{\eta} \bigg).
\end{align}
Using \eqref{eq.reg.41} and the selection of $\delta_1$, we can bound \eqref{eq.reg.7} by
\begin{equation} \label{eq.reg.71}
	\sum_{\tau=1}^T \mathbb{E}_{a_{t(\tau)}} \Big[ (\tilde{\bm{p}}_\tau - \tilde{\bm{p}} )^\top \tilde{\bm{l}}_\tau \Big]  \leq \sum_{\tau=1}^T \mathbb{E}_{a_{t(\tau)}} \Big[ (\tilde{\bm{p}}_\tau - \tilde{\bm{p}} )^\top \tilde{\bm{c}}_\tau \Big]+ \frac{KT}{\delta_1 (1-\delta_2 - \eta \delta_1)^{4\bar{d}}} = {\cal O}\bigg(\eta \bar{d} KT  + \frac{1+ \ln K}{\eta}  \bigg).
\end{equation}
Using \eqref{eq.reg.B1} and the selection of $\delta_1$, we have
\begin{align} \label{eq.reg.81}
	\sum_{\tau = 1}^T \mathbb{E}_{a_{t(\tau)}} \Big[ \big( \tilde{\bm{p}}_{\tau - \tilde{s}_\tau}  - \tilde{\bm{p}}_{\tau} \big)^\top \tilde{\bm{l}}_\tau \Big]  \leq \sum_{\tau = 1}^T \mathbb{E}_{a_{t(\tau)}} \Big[ \big( \tilde{\bm{p}}_{\tau - \tilde{s}_\tau}  - \tilde{\bm{p}}_{\tau} \big)^\top \tilde{\bm{c}}_\tau \Big] + \frac{KD}{\delta_1 (1-\delta_2 - \eta \delta_1)^{6\bar{d}}} = {\cal O} \big( \eta \bar{d} KD + \delta_2 KD\big).
\end{align}

Plugging \eqref{eq.reg.A1}, \eqref{eq.reg.71} , and \eqref{eq.reg.81} into \eqref{eq.reg.1}, the regret is bounded by 
\begin{equation}
		\text{Reg}_T=\sum_{t=1}^T\mathbb{E} \big[ \bm{p}_t^\top \bm{l}_t \big]- \sum_{t=1}^T\bm{p}^{*\top} \bm{l}_t = {\cal O} \big( \sqrt{ \bar{d}(T+D) K (1+\ln K)} \big).
\end{equation}

\section{Proofs for DBGD}
\subsection{Proof of Lemma \ref{lemma.Gbound} }

	Since $f_{s|t}(\cdot)$ is $L$-Lipschitz, we have $g_{s|t}(k) \leq \frac{1}{\delta}  L\| \delta \bm{e}_k \| =L$, and thus $\| \bm{g}_{s|t}\| \leq \sqrt{K}L$. On the other hand, let $\bm{\nabla}_{s|t}: = \nabla f_{s|t}(\bm{x}_{s|t})$, and $\nabla_{s|t}(k)$ being the $k$-th entry of $\bm{\nabla}_{s|t}$. Due to the $\beta$-smoothness of $f_{s|t} (\cdot)$, we have
	\begin{align}
		g_{s|t}(k) - \nabla_{s|t} (k) \leq \frac{1}{\delta} \big( \delta \bm{\nabla}_{s|t}^\top \bm{e}_k + \frac{\beta}{2} \delta^2 \big) - \nabla_{s|t} (k) = \frac{\beta \delta}{2}
	\end{align}
	suggesting that $\| \bm{g}_{s|t} - \nabla f_{s|t}(\bm{x}_{s|t}) \| \leq \frac{\beta \delta}{2}\sqrt{K}$.

\subsection{ Proof of Lemma \ref{lemma.descent} }
\textbf{Lemma \ref{lemma.descent}}  (Restate).
\textit{In virtual slots, it is guaranteed to have
	\begin{equation}
		\| \tilde{\bm{x}}_\tau - \tilde{\bm{x}}_{\tau - \tilde{s}_\tau} \| \leq \eta \tilde{s}_\tau \sqrt{K}L
	\end{equation}
	and for any $\bm{x} \in {\cal X}_\delta$, we have
	\begin{equation}
		\eta \tilde{\bm{g}}_\tau^\top \big( \tilde{\bm{x}}_\tau - \bm{x} \big) \leq \frac{\eta^2}{2} KL^2 +\frac{ \big{\|} \tilde{\bm{x}}_{\tau} - \bm{x} \big{\|}^2 - \big{\|} \tilde{\bm{x}}_{\tau+1} - \bm{x} \big{\|}^2}{2}. 
	\end{equation}}
\begin{proof}

	The proof begins with
	\begin{align}
		\| \tilde{\bm{x}}_{\tau - \tilde{s}_\tau} - \tilde{\bm{x}}_\tau \| \leq \sum_{j=0}^{\tilde{s}_\tau -1}\|\tilde{\bm{x}}_{\tau - \tilde{s}_\tau + j} - \tilde{\bm{x}}_{\tau - \tilde{s}_\tau + j +1} \| \stackrel{(a)}{\leq} \eta \tilde{s}_\tau \sqrt{K}L
	\end{align}
	where (a) uses the fact that $\|\tilde{\bm{x}}_\tau - \tilde{\bm{x}}_{\tau+1}\| = \big{\|} \tilde{\bm{x}}_\tau - \Pi_{{\cal X}_\delta} [\tilde{\bm{x}}_\tau - \eta \tilde{\bm{g}}_\tau] \big{\|} \leq \eta \|\tilde{\bm{g}}_\tau \|  $. The first inequality is thus proved
	
	Then, notice that 
	\begin{align}\label{eq.descent.1}
		\big{\|} \tilde{\bm{x}}_{\tau+1} - \bm{x} \big{\|}^2 - \big{\|} \tilde{\bm{x}}_{\tau} - \bm{x} \big{\|}^2 &= \big{\|} \Pi_{{\cal X}_\delta} [\tilde{\bm{x}}_\tau - \eta \tilde{\bm{g}}_\tau] - \bm{x} \big{\|}^2 - \big{\|} \tilde{\bm{x}}_{\tau} - \bm{x} \big{\|}^2 \nonumber \\
		& \stackrel{(b)}{\leq} \big{\|} \tilde{\bm{x}}_\tau- \bm{x} - \eta \tilde{\bm{g}}_\tau  \big{\|}^2 - \big{\|} \tilde{\bm{x}}_{\tau} - \bm{x} \big{\|}^2  = - 2  \eta \tilde{\bm{g}}_\tau^\top \big( \tilde{\bm{x}}_\tau - \bm{x} \big) + \eta^2 \big{\|} \tilde{\bm{g}}_\tau \big{\|}^2
	\end{align}
	where inequality (b) uses the non-expansion property of projection. Rearranging the terms of \eqref{eq.descent.1} completes the proof.
\end{proof}


\subsection{ Proof of Theorem \ref{theo.reg2} }

\begin{lemma}\label{lemma.hproper}
	 Let $h_t (\bm{x}):= f_t(\bm{x}) + \big( \bm{g}_t - \nabla f_t (\bm{x}_t) \big)^\top \bm{x}$, where $\bm{g}_t: = \bm{g}_{t|t+d_t}$. Then
	  $h_t (\bm{x})$ has the following properties: i) $h_t(\bm{x})$ is $\big(L+\frac{\beta\delta \sqrt{K}}{2}\big)$-Lipschitz; and ii) $h_t(\bm{x})$ is $\beta$ smooth and convex.
\end{lemma}
\begin{proof}
	Starting with the first property, consider that 
	\begin{align}
		\| h_t(\bm{x}) - h_t(\bm{y}) \| &= \big{\|} f_t(\bm{x}) + \big( \bm{g}_t - \nabla f_t (\bm{x}_t) \big)^\top \bm{x} - f_t(\bm{y}) - \big( \bm{g}_t - \nabla f_t (\bm{x}_t) \big)^\top \bm{y} \big{\|} \nonumber \\
		& \leq \|f_t(\bm{x}) - f_t(\bm{y}) \| + \big{\|} \bm{g}_t - \nabla f_t (\bm{x}_t)  \big{\|} \| \bm{x}-\bm{y} \| \stackrel{(a)}{\leq} \bigg(L+ \frac{\beta \delta \sqrt{K}}{2} \bigg) \|\bm{x}-\bm{y} \|
	\end{align}
	where in (a) we used the results in Lemma \ref{lemma.Gbound}.
For the second property, the convexity of $h_t(\bm{x})$ is obvious. Then noticing that $\nabla h_t(\bm{x}) = \nabla f_t(\bm{x})+ \bm{g}_t - \nabla f_t(\bm{x}_t)$, we have 
	\begin{align}
		h_t (\bm{y}) - h_t (\bm{x}) &= f_t(\bm{y}) - f_t(\bm{x}) + \big( \bm{g}_t - \nabla f_t (\bm{x}_t) \big)^\top (\bm{y} - \bm{x}) \nonumber \\
		& \leq \big(\nabla f_t(\bm{x})\big)^\top (\bm{y} - \bm{x}) + \frac{\beta}{2} \|\bm{y} - \bm{x}\|^2 + \big( \bm{g}_t - \nabla f_t (\bm{x}_t) \big)^\top (\bm{y} - \bm{x}) \nonumber \\
		& = \big( \nabla h_t(\bm{x}) \big)^\top (\bm{y} - \bm{x}) + \frac{\beta}{2} \|\bm{y} - \bm{x}\|^2
	\end{align}
which implies that $h_t(\bm{x})$ is $\beta$ smooth.
\end{proof}

Then we are ready to prove Theorem \ref{theo.reg2}. Let $h_t (\bm{x}):= f_t(\bm{x}) + \big( \bm{g}_t - \nabla f_t (\bm{x}_t) \big)^\top \bm{x}$, where $\bm{g}_t: = \bm{g}_{t|t+d_t}$. Using the property of $h_t (\bm{x})$ in Lemma \ref{lemma.hproper} as well as the fact $\nabla h_t (\bm{x}_t) = \bm{g}_t$, we have
	\begin{align}\label{eq.reg2.1}
		\text{Reg}_T &= \sum_{t=1}^T f_t(\bm{x}_t) - \sum_{t=1}^T f_t(\bm{x}^*)\nonumber\\
		& = \sum_{t=1}^T  \bigg( h_t (\bm{x}_t) - \big( \bm{g}_t - \nabla f_t (\bm{x}_t) \big)^\top \bm{x}_t \bigg) - \sum_{t=1}^T \bigg( h_t (\bm{x}^*) - \big( \bm{g}_t - \nabla f_t (\bm{x}_t) \big)^\top \bm{x}^* \bigg) \nonumber \\
		& = \sum_{t=1}^T  \bigg( h_t (\bm{x}_t)- h_t (\bm{x}^*) \bigg) + \sum_{t=1}^T \big( \bm{g}_t - \nabla f_t (\bm{x}_t) \big)^\top \big( \bm{x}^* - \bm{x}_t \big) \nonumber \\
		& \stackrel{(a)}{\leq} \sum_{t=1}^T \bigg( h_t (\bm{x}_t)- h_t (\bm{x}_\delta) \bigg) + \sum_{t=1}^T \bigg( h_t (\bm{x}_\delta)- h_t (\bm{x}) \bigg) + \frac{RT\beta \delta \sqrt{K}}{2} \nonumber \\
		& \stackrel{(b)}{\leq} \sum_{t=1}^T \bigg( h_t (\bm{x}_t)- h_t (\bm{x}_\delta) \bigg) + \delta R T \Big(L+\frac{\beta \delta\sqrt{K}}{2}\Big) + \frac{RT\beta \delta \sqrt{K}}{2} 
	\end{align}
	where in (a) $\bm{x}_\delta:= \Pi_{{\cal X}_\delta} (\bm{x}^*)$, and the inequality follows from the results in Lemma \ref{lemma.Gbound}; (b) follows from the fact that $h_t (\cdot)$ is $(L+ \frac{\beta\delta\sqrt{K}}{2})$-Lipschitz, as well as $\|\bm{x}_\delta - \bm{x} \| \leq \delta R$.

	Hence, at virtual slots, it is like learning according to $h_t(\bm{x}_t)$, with $\nabla h_t(\bm{x}_t)$ being revealed. 
	With the short-hand notation $\tilde{h}_\tau(\cdot) := h_{t(\tau)}(\cdot) $, we have (using similar arguments like the proof of Theorem \ref{theo.reg})
	\begin{align}\label{eq.reg2.4}
		\sum_{t=1}^T  h_t (\bm{x}_t)-  \sum_{t=1}^T h_t (\bm{x}_\delta) & = \sum_{\tau=1}^T h_{t(\tau)} (\bm{x}_{t(\tau)}) -\sum_{\tau=1}^T h_{t(\tau)} (\bm{x}_\delta) = \sum_{\tau=1}^T \tilde{h}_\tau (\tilde{\bm{x}}_{\tau - \tilde{s}_\tau}) -\sum_{\tau=1}^T \tilde{h}_\tau (\bm{x}_\delta) \nonumber \\
		&= \sum_{\tau=1}^T \tilde{h}_\tau (\tilde{\bm{x}}_{\tau - \tilde{s}_\tau}) -  \sum_{\tau=1}^T \tilde{h}_\tau (\tilde{\bm{x}}_\tau) + \sum_{\tau=1}^T \tilde{h}_\tau (\tilde{\bm{x}}_\tau) - \sum_{\tau=1}^T \tilde{h}_\tau (\bm{x}_\delta). 
	\end{align}	
	
The first term in the RHS of \eqref{eq.reg2.4}	can be bounded as
	\begin{align}
		\tilde{h}_\tau (\tilde{\bm{x}}_{\tau - \tilde{s}_\tau}) - \tilde{h}_\tau (\tilde{\bm{x}}_\tau) \leq  \big{\|} \tilde{h}_\tau (\tilde{\bm{x}}_{\tau - \tilde{s}_\tau}) - \tilde{h}_\tau (\tilde{\bm{x}}_\tau)\big{\|} \stackrel{(c)}{\leq} \bigg( L + \frac{\beta \delta \sqrt{K}}{2} \bigg) \big{\|} \tilde{\bm{x}}_{\tau - \tilde{s}_\tau} - \tilde{\bm{x}}_\tau \big{\|} \stackrel{(d)}{\leq} \eta \tilde{s}_\tau \sqrt{K}L\bigg( L + \frac{\beta \delta \sqrt{K}}{2} \bigg)
	\end{align}
	where (c) follows from Lemma \ref{lemma.hproper}; and (d) is the result of Lemma  \ref{lemma.descent}. Hence, using $\sum_{\tau=1}^T \tilde{s}_\tau = D$ in Lemma \ref{lemma.index}, we obtain
	\begin{equation}\label{eq.reg2.2}
		\sum_{\tau=1}^T \tilde{h}_\tau (\tilde{\bm{x}}_{\tau - \tilde{s}_\tau}) -  \sum_{\tau=1}^T \tilde{h}_\tau (\tilde{\bm{x}}_\tau) \leq \eta D \sqrt{K}L\bigg( L + \frac{\beta \delta \sqrt{K}}{2} \bigg).
	\end{equation}
	
	On the other hand, by the convexity of $\tilde{h}_\tau(\cdot)$, we have
	\begin{align}
		\tilde{h}_\tau(\tilde{\bm{x}}_\tau) - \tilde{h}_\tau(\bm{x}_\delta) &\leq \big( \nabla \tilde{h}_\tau(\tilde{\bm{x}}_\tau) \big)^\top \big( \tilde{\bm{x}}_\tau - \bm{x}_\delta \big)  = \big[\nabla \tilde{h}_\tau(\tilde{\bm{x}}_\tau) - \tilde{\bm{g}}_\tau \big]^\top \big( \tilde{\bm{x}}_\tau - \bm{x}_\delta \big) + \tilde{\bm{g}}_\tau^\top \big( \tilde{\bm{x}}_\tau - \bm{x}_\delta \big) \nonumber \\
		& \stackrel{(e)}{\leq} \beta\big{\|} \tilde{\bm{x}}_\tau - \tilde{\bm{x}}_{\tau - \tilde{s}_\tau} \big{\|} \big{\|} \tilde{\bm{x}}_\tau - \bm{x}_\delta \big{\|} + \tilde{\bm{g}}_\tau^\top \big( \tilde{\bm{x}}_\tau - \bm{x}_\delta \big) \leq \beta R \big{\|} \tilde{\bm{x}}_\tau - \tilde{\bm{x}}_{\tau - \tilde{s}_\tau} \big{\|}  + \tilde{\bm{g}}_\tau^\top \big( \tilde{\bm{x}}_\tau - \bm{x}_\delta \big)
	\end{align}
	where (e) is because $\tilde{h}_\tau (\cdot)$ is $\beta$-smoothness [cf. \cite[Thm 2.1.5]{nesterov1998}]. Taking summation over $\tau$ and leveraging the results in Lemma \ref{lemma.descent}, we have 
	\begin{align}\label{eq.reg2.3}
		\sum_{\tau=1}^T \tilde{h}_\tau(\tilde{\bm{x}}_\tau) - \tilde{h}_\tau(\bm{x}_\delta) & \leq \sum_{\tau=1}^T \eta \tilde{s}_\tau \sqrt{K}L \beta R + \sum_{\tau=1}^T \frac{\eta }{2}  \big{\|} \tilde{\bm{g}}_\tau \big{\|}^2 + \frac{R^2}{\eta} \leq  \eta D \sqrt{K}L \beta R + \frac{\eta T}{2} KL^2 +\frac{R^2}{\eta}.
	\end{align}

Selecting $\delta = {\cal O} \big(  1/(T+D) \big)$, \eqref{eq.reg2.2} implies
\begin{equation}\label{eq.reg2.21}
	\sum_{\tau=1}^T \tilde{h}_\tau (\tilde{\bm{x}}_{\tau - \tilde{s}_\tau}) -  \sum_{\tau=1}^T \tilde{h}_\tau (\tilde{\bm{x}}_\tau) \leq \eta D \sqrt{K}L\bigg( L + \frac{\beta \delta \sqrt{K}}{2} \bigg) = {\cal O} \big( \eta \sqrt{K} D \big).
\end{equation}

Inequality \eqref{eq.reg2.3} then becomes 
\begin{align}\label{eq.reg2.31}
	\sum_{\tau=1}^T \tilde{h}_\tau(\tilde{\bm{x}}_\tau) - \tilde{h}_\tau(\bm{x}_\delta)  \leq  \eta D \sqrt{K}L \beta R + \frac{\eta T}{2} KL^2 +\frac{R^2}{\eta} = {\cal O} \bigg( \eta K T+\eta \sqrt{K} D + \frac{1}{\eta} \bigg).
\end{align}
Plugging \eqref{eq.reg2.4}, \eqref{eq.reg2.2}, and \eqref{eq.reg2.3} into \eqref{eq.reg2.1}, and choosing $\eta = {\cal O}(1/\sqrt{K(T+D)})$, the proof is complete. 



\subsection{Proof of Corollary \ref{coro.reg}}

To prove Corollary \ref{coro.reg}, we will show that
	\begin{equation}\label{eq.coro.1}
	   \frac{1}{K+1}\sum_{t=1}^T\sum_{k=0}^{K} f_t(\bm{x}_{t,k}) - \sum_{t=1}^T f_t(\bm{x}_t) = {\cal O}(\sqrt{K}).
	\end{equation}

Using the $\beta$-smoothness in Assumption 4, we have for any $k \neq 0$
	\begin{align}\label{eq.coro.3}
		f_t(\bm{x}_{t,k}) - f_t(\bm{x}_t) \leq \big( \nabla f_t(\bm{x}_t) \big)^\top (\bm{x}_{t,k} - \bm{x}_t) + \frac{\beta\delta^2}{2} \leq \delta \| \nabla f_t(\bm{x}_t) \| + \frac{\beta\delta^2}{2}.
	\end{align}
	Then leveraging the result of Lemma \ref{lemma.Gbound}, we have
	\begin{align}\label{eq.coro.2}
		\| \nabla f_t(\bm{x}_t) \| &= \| \nabla f_{t|t+d_t}(\bm{x}_{t|t+d_t}) \| = \| \nabla f_{t|t+d_t}(\bm{x}_{t|t+d_t}) + \bm{g}_{t|t+d_t} - \bm{g}_{t|t+d_t} \|  \nonumber \\
		& \leq \|\bm{g}_{t|t+d_t} \| +  \| \nabla f_{t|t+d_t}(\bm{x}_{t|t+d_t}) - \bm{g}_{t|t+d_t}\| \leq \sqrt{K}L + \frac{\beta\delta\sqrt{K}}{2}.
	\end{align}
	Plugging \eqref{eq.coro.2} back to \eqref{eq.coro.3}, we have
	\begin{align}
		f_t(\bm{x}_{t,k}) - f_t(\bm{x}_t) \leq \delta \sqrt{K}L + \frac{\beta\delta^2\sqrt{K}}{2} + \frac{\beta\delta^2}{2} \stackrel{(a)}{=} {\cal O} \Big(\frac{\sqrt{K}}{T+D} \Big)
	\end{align}
	where (a) follows from $\delta= {\cal O}\big((T+D)^{-1}\big)$. Summing over $k$ and $t$  readily implies \eqref{eq.coro.1}.

\end{document}